\documentclass[conference]{IEEEtran}
\IEEEoverridecommandlockouts
\usepackage{cite}
\usepackage{amsmath,amssymb,amsfonts,amsthm}
\usepackage{enumerate}
\usepackage{algorithm}
\usepackage{algorithmic}
\usepackage{graphicx}
\usepackage{textcomp}
\usepackage{xcolor}
\def\BibTeX{{\rm B\kern-.05em{\sc i\kern-.025em b}\kern-.08em
		T\kern-.1667em\lower.7ex\hbox{E}\kern-.125emX}}

\def\bx{{\boldsymbol x}}

\begin{document}
	
	\title{ Approximate Message Passing for Multi-Layer Estimation in Rotationally Invariant Models} %\\
	%\title{ Approximate Message Passing for Multi Layer Estimation} %\\
	%	\thanks{Identify applicable funding agency here. If none, delete this.}
	%}
	
	\author{\IEEEauthorblockN{Yizhou Xu$^{\dagger}$, TianQi Hou$^{\diamond}$, ShanSuo Liang$^{ \diamond }$  and Marco Mondelli$^*$\\
			$\dagger$ School of Aerospace Engineering, Tsinghua University, China  \\
			$\diamond$ Theory Lab, Central Research Institute, 2012 Labs, Huawei Technologies Co., Ltd.\\
			$ * $  Institute of Science and Technology Austria (ISTA)\\
		} Emails:  xu-yz19@mails.tsinghua.edu.cn, thou@connect.ust.hk, liang.shansuo@huawei.com, marco.mondelli@ist.ac.at}
	
	%	\author{\IEEEauthorblockN{1\textsuperscript{st} Given Name Surname}
	%		\IEEEauthorblockA{\textit{dept. name of organization (of Aff.)} \\
	%			\textit{name of organization (of Aff.)}\\
	%			City, Country \\
	%			email address or ORCID}
	%		\and
	%		\IEEEauthorblockN{2\textsuperscript{nd} Given Name Surname}
	%		\IEEEauthorblockA{\textit{dept. name of organization (of Aff.)} \\
	%			\textit{name of organization (of Aff.)}\\
	%			City, Country \\
	%			email address or ORCID}
	%		\and
	%		\IEEEauthorblockN{3\textsuperscript{rd} Given Name Surname}
	%		\IEEEauthorblockA{\textit{dept. name of organization (of Aff.)} \\
	%			\textit{name of organization (of Aff.)}\\
	%			City, Country \\
	%			email address or ORCID}
	%		\and
	%		\IEEEauthorblockN{4\textsuperscript{th} Given Name Surname}
	%		\IEEEauthorblockA{\textit{dept. name of organization (of Aff.)} \\
	%			\textit{name of organization (of Aff.)}\\
	%			City, Country \\
	%			email address or ORCID}
	%		\and
	%		\IEEEauthorblockN{5\textsuperscript{th} Given Name Surname}
	%		\IEEEauthorblockA{\textit{dept. name of organization (of Aff.)} \\
	%			\textit{name of organization (of Aff.)}\\
	%			City, Country \\
	%			email address or ORCID}
	%		\and
	%		\IEEEauthorblockN{6\textsuperscript{th} Given Name Surname}
	%		\IEEEauthorblockA{\textit{dept. name of organization (of Aff.)} \\
	%			\textit{name of organization (of Aff.)}\\
	%			City, Country \\
	%			email address or ORCID}
	%	}
	
	\maketitle
	
	\begin{abstract}
		We consider the problem of reconstructing the signal and the hidden variables from observations coming from a multi-layer network with rotationally invariant weight matrices. The multi-layer structure models inference from deep generative priors, and the rotational invariance imposed on the weights generalizes the i.i.d.\ Gaussian assumption by allowing for a complex correlation structure, which is typical in applications. In this work, we present a new class of approximate message passing (AMP) algorithms and give a state evolution recursion which precisely characterizes their performance in the large system limit. In contrast with the existing multi-layer VAMP (ML-VAMP) approach, our proposed AMP -- dubbed multi-layer rotationally invariant generalized AMP (ML-RI-GAMP) -- provides a natural generalization beyond Gaussian designs, in the sense that it recovers the existing Gaussian AMP as a special case. Furthermore, ML-RI-GAMP exhibits a significantly lower complexity than ML-VAMP, as the computationally intensive singular value decomposition is replaced by an estimation of the moments of the design matrices. Finally, our numerical results show that this complexity gain comes at little to no cost in the performance of the algorithm.
		%		We consider the problem of reconstructing the signal from multi-layer non-linear measurements with rotationally invariant design matrices, arising in variable applications since such matrix ensemble can capture more complex correlation structures compared with the i.i.d ones. In this work, we propose an approximate message passing (AMP) algorithm to estimate the signal in the multi-later setting  called the multi-layer rotation invariant general AMP  (ML-RI-GAMP), and then rigorously prove that its performance  is precisely characterized by its state evolution  in the  large system limit.  Numerical results show that it reaches a  performance very close to the multi-layer VAMP algorithm (which is conjectured to be  Bayes-optimal in the multi-layer setting) but with much lower complexity as the singular value decomposition is replaced by estimating free cumulants of the design matrices in each layer. 

	\end{abstract}
	% ---------------------------------------------------
	\section{Introduction}
	
	We consider inference in an $L$-layer network of the form 
	%	In the multi-layer inference, We consider the estimation problem of the form
	\begin{equation}
		\boldsymbol{g}^\ell=\boldsymbol{A}_\ell\boldsymbol{x}^\ell, \quad\quad  \boldsymbol{x}^{\ell+1}=q^\ell(\boldsymbol{g}^\ell,\boldsymbol{\epsilon}^{\ell+1}).\label{model}
	\end{equation}
	Here, for $\ell\in [L]:=\{1,\ldots, L\}$, $\boldsymbol{x}^\ell\in\mathbb{R}^{n_\ell}$ is the feature vector in the $\ell$-th layer (when $\ell=1$, $\boldsymbol{x}^1$ is the input signal),  $\boldsymbol{A}_\ell\in\mathbb{R}^{n_{\ell+1}\times n_\ell}$ is the \emph{known} design matrix, $q^\ell$ is a \emph{known} function applied component-wise, and $\boldsymbol{\epsilon}^{\ell+1}$ is an \emph{unknown} i.i.d.\ noise vector. Given the observation $\boldsymbol{y}=\boldsymbol{x}^{L+1}$ (and the knowledge of $\{\boldsymbol{A}_\ell\}_{\ell=1}^L$ and $\{q^\ell\}_{\ell=1}^L$), the goal is to estimate the features $\{\boldsymbol{x}^\ell\}_{\ell=1}^L$. This problem models inference with deep generative priors \cite{bora2017compressed,yeh2016semantic} and is motivated by the exceptional practical success of generative models, e.g., variational auto-encoders \cite{rezende2014stochastic} and generative adversarial networks\cite{radford2015unsupervised}. % and deep image priors \cite{ulyanov2018deep}. 
	In this context, the input $\boldsymbol{x}^1$ is a noise realization and the layers correspond to a noisy measurement process which leads to the output $\boldsymbol{y}$. The original data/image can then be recovered from one of the feature vectors in the hidden layers. 
	
	The inference task above is often solved via gradient-based methods \cite{bora2017compressed,yeh2016semantic}. However, due to the non-convexity of the objective, little is known about how these methods perform. To address the issue, a recent line of work has focused on a class of algorithms known as approximate message passing (AMP). In fact, a key feature of AMP is that, under suitable assumptions, its performance in the large system limit is characterized by a deterministic recursion called state evolution. % \cite{bayati2011dynamics}. %,bolthausen2014iterative}. 
	Originally proposed for estimation in linear models \cite{DMM09}, AMP algorithms have then been applied to a wide range of estimation problems, including generalized linear models \cite{rangan2011generalized} and multi-layer models \cite{manoel2017multi} %\cite{aubin2020exact,manoel2017multi} 
	(see also the recent review \cite{feng2022unifying} and the references therein).
	However, these works all assume the design matrices to be i.i.d. Gaussian, which limits their applicability in practice. Here, we assume that $\{\boldsymbol{A}_\ell\}_{\ell=1}^L$ are rotationally invariant. This imposes that the orthogonal matrices in their singular value decomposition (SVD) are uniformly random, but it allows for arbitrary spectra. Hence, the design matrices are well equipped to capture the complex correlation structure typical in applications. One approach to deal with rotationally invariant designs is based on expectation propagation, and it has led to the class of Vector AMP (VAMP) algorithms for inference in the linear \cite{rangan2019vector,takeuchi2021convergence} and multi-layer setting \cite{pandit2020inference1,gabrie2018entropy}.
	
	In this paper, we derive a new class of AMP algorithms for inference in the model \eqref{model} with rotationally invariant design matrices, and we provide a rigorous performance characterization in the high-dimensional limit via a state evolution analysis. Our proposed multi-layer rotationally invariant generalized AMP (ML-RI-GAMP) has two key advantages over the existing multi-layer VAMP (ML-VAMP) \cite{pandit2020inference1}. On the theoretical side, ML-RI-GAMP constitutes a natural generalization of the AMP for Gaussian designs, since the Gaussian AMP \cite{manoel2017multi} is recovered as a special case. %We remark that this feature can potentially be used to capture the effect of mismatch in the noise statistics \cite{barbier2022price}.
	On the practical side, unlike ML-VAMP, ML-RI-GAMP does not need a computationally expensive SVD (with cubic complexity in the ambient dimension). In fact, it suffices to estimate the moments of the eigenvalue distributions of the design matrices, which can be performed with quadratic complexity. We highlight that this reduction in the computational complexity does not result in a performance loss: our numerical simulations show that ML-RI-GAMP exhibits a performance close to ML-VAMP, which was shown in \cite{pandit2020inference1} to coincide with the replica prediction for the Bayes-optimal error (whenever the state evolution recursion has a unique fixed point). 	At the technical level, we follow the approach put forward in \cite{fan2022approximate} for low-rank matrix estimation and recently extended to generalized linear models in \cite{venkataramanan2022estimation}. % At the technical level, the state evolution for our proposed ML-RI-GAMP is obtained by reducing to the abstract AMP iteration in \cite{fan2022approximate}.

	\section{ML-RI-GAMP algorithm}\label{sec:AMP}
	
	Our proposed ML-RI-GAMP algorithm is described in Algorithm \ref{algo:MLRIGAMP}. Lines 1-4 represent the initialization; Lines 7-9 and Line 12 represent the core part of the algorithm.
	The first and last layers need to be treated slightly differently, as shown in Lines 3-5 and 14. The functions $f_t^1:\mathbb{R}^{t}\to\mathbb{R}$, $h_t^L:\mathbb{R}^{t+1}\to\mathbb{R}$, $f_t^\ell:\mathbb{R}^{2t}\to\mathbb{R}$ ($\ell \in \{2, \ldots, L\}$), $h_t^\ell:\mathbb{R}^{2t}\to\mathbb{R}$ ($\ell \in \{1, \ldots, L-1\}$) are Lipschitz and applied component-wise. We note that these functions -- typically called \emph{denoisers} -- are arbitrary (modulo mild technical requirements), and we will optimize them in Section \ref{sec:impl} in order to minimize the mean squared error. % can be arbitary in theory, but they are suggested to be minimum mean squared error estimators to obtain better performances.
	
	The \emph{Onsager coefficients} $\{\alpha_{ti}^\ell\}_{i=1}^t$ and $\{\beta_{ti}^\ell\}_{i=1}^{t-1}$ are designed to de-bias the iterates $\{\boldsymbol{x}_t^\ell, \boldsymbol{r}_t^\ell\}_{\ell=1}^L$ so that -- after conditioning on the feature vectors $\{\boldsymbol{x}^\ell\}_{\ell=1}^L$ -- their joint empirical distributions converge to a multi-variate Gaussian, whose covariance is then tracked via state evolution. We compute these coefficients via several auxiliary matrices. For $\ell\in [L]$, let
	%	to jointly Gaussian distribution shown by the SE are calculated according to following auxiliary matrices. Such coefficients . Auxiliary matrices 
	$\boldsymbol{\Psi}_{t+1}^\ell,\boldsymbol{\Phi}_{t+1}^\ell\in\mathbb{R}^{(t+1)\times(t+1)}$ be defined as %similarly as in \cite{venkataramanan2022estimation}: 
	\begin{equation}
		\boldsymbol{\Psi}_{t+1}^\ell=\left(\begin{matrix}
			0 & 0  & \cdots   & 0 & 0 \\
			0 & \langle\partial_1\hat{\boldsymbol{x}}_1^\ell\rangle  & 0   & \cdots & 0\\
			0 & \langle\partial_1\hat{\boldsymbol{x}}_2^\ell\rangle  & \langle\partial_2\hat{\boldsymbol{x}}_2^\ell\rangle   & \cdots & 0\\
			\vdots & \vdots  & \vdots  & \ddots   & \vdots & \\
			0 & \langle\partial_1\hat{\boldsymbol{x}}_t^\ell\rangle  &\langle\partial_2\hat{\boldsymbol{x}}_t^\ell\rangle& \cdots\  & \langle\partial_t\hat{\boldsymbol{x}}_t^\ell\rangle & \\
		\end{matrix}\right),
	\end{equation}
	\begin{equation}
		\boldsymbol{\Phi}_{t+1}^\ell=\left(\begin{matrix}
			0 & 0  & \cdots   & 0 & 0 \\
			\langle\partial_g\boldsymbol{s}_1^\ell\rangle & 0  & 0   & \cdots & 0\\
			\langle\partial_g\boldsymbol{s}_2^\ell\rangle & \langle\partial_1\boldsymbol{s}_2^\ell\rangle  & 0   & \cdots & 0\\
			\vdots & \vdots  & \vdots  & \ddots   & \vdots & \\
			\langle\partial_g\boldsymbol{s}_t^\ell\rangle & \langle\partial_1\boldsymbol{s}_t^\ell\rangle & \cdots\  & \langle\partial_{t-1}\boldsymbol{s}_t^\ell\rangle & 0\\
		\end{matrix}\right).
	\end{equation}
	Here, given a vector $\boldsymbol{x}=(x_1, \ldots, x_{n_\ell})\in\mathbb{R}^{n_\ell}$,  $\langle\boldsymbol{x}\rangle$ denotes its empirical average $\frac{1}{n_\ell}\sum_{i=1}^{n_\ell}x_i$. For $k\leq t$, the vector $\partial_k\hat{\boldsymbol{x}}_t^\ell \in \mathbb R^{n_\ell}$ denotes the partial derivative $\partial_{x_k^\ell}f_t^\ell(x_1^\ell,\ldots,x_t^\ell,r_1^{\ell-1},\ldots,r_t^{\ell-1})$ (resp.\ $\partial_{x_k^1}f_t^1(x_1^1,\ldots,x_t^1)$ for the first layer) applied row-wise to $\hat{\boldsymbol{x}}_t^\ell=f_t^\ell(\boldsymbol{x}_1^\ell,\ldots,\boldsymbol{x}_t^\ell,\boldsymbol{r}_1^{\ell-1},\ldots\boldsymbol{r}_t^{\ell-1})$ (resp.\ $\hat{\boldsymbol{x}}_t^1=f_t^1(\boldsymbol{x}_1^1,\ldots,\boldsymbol{x}_t^1)$). Similarly, $\partial_k\boldsymbol{s}_t^\ell \in \mathbb R^{n_{\ell+1}}$ denotes the partial derivative $\partial_{r_k^\ell}h_{t}^\ell(r_1^\ell,\ldots,r_{t-1}^\ell,x_1^{\ell+1},\ldots,
	x_{t-1}^{\ell+1})$ (resp.\ $\partial_{r_k^L}h_{t}^L(r_1^L,\ldots,r_{t-1}^L)$ for the last layer) applied row-wise to $\boldsymbol{s}_{t}^\ell=h_{t}^\ell(\boldsymbol{r}_1^\ell,\ldots,\boldsymbol{r}_{t-1}^\ell,\boldsymbol{x}_1^{\ell+1},\ldots, \boldsymbol{x}_{t-1}^{\ell+1})$ (resp.\ $\boldsymbol{s}_{t}^L=h_{t}^\ell(\boldsymbol{r}_1^L,\ldots,\boldsymbol{r}_{t-1}^L,\boldsymbol{y})$). Furthermore, $\partial_g\boldsymbol{s}_t^\ell$ denotes the partial derivative $\partial_{g^\ell}h_{t}^\ell(r_1^\ell,\ldots,r_{t-1}^\ell,x_1^{\ell+1},\ldots,
	x_{t-1}^{\ell+1})$ applied row-wise to $\boldsymbol{s}_{t}^\ell=h_{t}^\ell(\boldsymbol{r}_1^\ell,\ldots,\boldsymbol{r}_{t-1}^\ell,\boldsymbol{x}_1^{\ell+1},\ldots, \boldsymbol{x}_{t-1}^{\ell+1})$ for $t>1$, which is computed as
	\begin{equation}
		\begin{aligned}
			\partial_{g^\ell}h_{t}^\ell(&r_1^\ell,...,r_{t-1}^\ell,x_1^{\ell+1},...,x_{t-1}^{\ell+1})=\sum_{i=1}^{t-1}\bar{\mu}_i^{\ell+1}\partial_{g^\ell}q^{\ell+1}(g^\ell,\epsilon^{\ell+1})\\
			&\cdot\partial_{x_i^{\ell+1}}h_{t}^\ell(r_1^\ell,\ldots,r_{t-1}^\ell,x_1^{\ell+1},...,x_{t-1}^{\ell+1}),
			\label{partial_gell}
		\end{aligned}
	\end{equation}
	where $\{\bar{\mu}_i^{\ell+1}\}_{i=1}^{t-1}$ is obtained from the state evolution recursion detailed in Section \ref{sec:SE}. We note that this definition is recursive, as $\partial_{g^\ell}h_{t}^\ell$ is calculated via the state evolution of step $1$ to $t-1$. For $t=1$, the partial derivative is defined as
	\begin{equation}
		\partial_{g^\ell}h_1^\ell=\partial_{g^\ell}h_1^\ell(q^\ell(g^\ell,\epsilon^{\ell+1}))\cdot\delta_{\ell+1}\kappa_{2}^{\ell+1}\partial_{g^{\ell+1}}h_1^{\ell+1},
	\end{equation}
	with the partial derivative of the last layer being simply $\partial_{g^L}h_{t}^\ell(r_1^\ell,\ldots,r_{t-1}^\ell,q^L(g^L,\epsilon^{L+1}))$.

	\begin{algorithm}[tb]
		\caption{ML-RI-GAMP}
		\begin{algorithmic}[1]\label{algo:MLRIGAMP}
			\REQUIRE %rotational invariant and mutually independent 
			design matrices $\{\boldsymbol{A}_\ell\}_{\ell=1}^L$, output $\boldsymbol{y}=\boldsymbol{x}^{L+1}$
			\STATE initialize: $\boldsymbol{s}_1^L=h_1^L(\boldsymbol{y})$
			\FOR {$\ell=L-1,\ldots,1$}
			\STATE $\boldsymbol{s}_1^\ell=h_1^\ell(\boldsymbol{A}_\ell^T\boldsymbol{s}_1^{\ell+1})$
			\ENDFOR
			\FOR {$t=1,\ldots,T$}
			\STATE$\boldsymbol{x}_t^1=\boldsymbol{A}_1^T\boldsymbol{s}_t^\ell-\sum_{i=1}^{t-1} \beta_{ti}^1\hat{\boldsymbol{x}}_i^1$
			\STATE$\hat{\boldsymbol{x}}_t^1=f_t^1(\boldsymbol{x}_1^1,\ldots,\boldsymbol{x}_t^1)$
			\STATE$\boldsymbol{r}_t^1=\boldsymbol{A}_1\hat{\boldsymbol{x}}_t^\ell-\sum_{i=1}^t\alpha_{ti}^1\boldsymbol{s}_i^1$
			\FOR {$\ell=2,\ldots,L$}
			\STATE$\boldsymbol{x}_t^\ell=\boldsymbol{A}_\ell^T\boldsymbol{s}_t^\ell-\sum_{i=1}^{t-1} \beta_{ti}^\ell\hat{\boldsymbol{x}}_i^\ell$
			\STATE$\hat{\boldsymbol{x}}_t^\ell=f_t^\ell(\boldsymbol{x}_1^\ell,\ldots,\boldsymbol{x}_t^\ell,\boldsymbol{r}_1^{\ell-1},\ldots,\boldsymbol{r}_t^{\ell-1})$
			\STATE$\boldsymbol{r}_t^\ell=\boldsymbol{A}_\ell\hat{\boldsymbol{x}}_t^\ell-\sum_{i=1}^t\alpha_{ti}^\ell\boldsymbol{s}_i^\ell$
			\ENDFOR
			\FOR {$\ell=1,\ldots,L-1$}
			\STATE$\boldsymbol{s}_{t+1}^\ell=h_{t+1}^\ell(\boldsymbol{r}_1^\ell,\ldots,\boldsymbol{r}_t^\ell,\boldsymbol{x}_1^{\ell+1},\ldots,
			\boldsymbol{x}_t^{\ell+1})$
			\ENDFOR
			\STATE $\boldsymbol{s}_{t+1}^L=h_{t+1}^\ell(\boldsymbol{r}_1^L,\ldots,\boldsymbol{r}_t^L,\boldsymbol{y})$
			\ENDFOR
			\RETURN final estimation $\{\hat{\boldsymbol{x}}_T^\ell\}_{\ell=1}^L$
		\end{algorithmic}
	\end{algorithm}
	
	Next, for $\ell\in [L]$, we define matrices $\boldsymbol{M}_{t+1}^{\alpha,\ell},\boldsymbol{M}_{t+1}^{\beta,\ell}\in\mathbb{R}^{(t+1)\times(t+1)}$ as
	\begin{equation}
		\boldsymbol{M}_{t+1}^{\alpha,\ell}=\sum_{j=0}^{t+1}\kappa_{2(j+1)}^\ell\boldsymbol{\Psi}^\ell_{t+1}(\boldsymbol{\Phi}_{t+1}^\ell\boldsymbol{\Psi}^\ell_{t+1})^j,
	\end{equation}
	\begin{equation}
		\boldsymbol{M}_{t+1}^{\beta,\ell}=\delta_\ell\sum_{j=0}^{t+1}\kappa_{2(j+1)}^\ell\boldsymbol{\Phi}^\ell_{t+1}(\boldsymbol{\Psi}^\ell_{t+1}\boldsymbol{\Phi}^\ell_{t+1})^j.\label{boldsymbolM_t+1}
	\end{equation}
	Here, for $\ell\in [L]$, $\{\kappa_{2k}^\ell\}_{k=1}^{t+1}$ denote the first $(t+1)$ rectangular free cumulants of $ \boldsymbol{A}_\ell$, which can be recursively computed from the first $(t+1)$ moments $\{m_{2k}^\ell\}_{k=1}^{t+1}$ of the eigenvalue distribution of  $ \boldsymbol{A}_\ell \boldsymbol{A}_\ell^{T}$ as \cite{venkataramanan2022estimation}
	\begin{equation}\label{eq:fc}
		\kappa_{2k}^\ell=m_{2k}^\ell-[z^k]\sum_{j=1}^{k-1}\kappa_{2j}^\ell(z(\delta_\ell M^\ell(z)+1)(M^\ell(z)+1))^j,
	\end{equation}
	where $M^\ell(z)=\sum_{k=0}^\infty m_{2k}^\ell z^k$ and $[z^k](q(z))$ denotes the coefficient of $z^k$
	in the polynomial $q(z)$. %\marco{either cite a precise equation in \cite{benaych2009rectangular}, or write down in appendix how to compute the free cumulants from the moments.} \yizhou{Now I write the formula here.} 
	We highlight that the moments $\{m_{2k}^\ell\}_{k=1}^{t+1}$ can be estimated in $O(n_{\ell+1}^2)$ time 
	(see \cite{liu2022memory,venkataramanan2022estimation}) and, hence, the complexity of estimating the free cumulants is of the same order as one iteration of our ML-RI-GAMP. 
	
	Finally, the coefficients $\{\alpha_{ti}^\ell\}_{i=1}^t$ and $\{\beta_{ti}^\ell\}_{i=1}^{t-1}$ are obtained from the last rows of $\boldsymbol{M}_{t+1}^{\alpha,\ell}$ and $\boldsymbol{M}_{t+1}^{\beta,\ell}$:
	\begin{equation}
		(\alpha_{t1}^\ell,\ldots,\alpha_{tt}^\ell)=[\boldsymbol{M}_{t+1}^{\alpha,\ell}]_{t+1,2:t+1},
	\end{equation}
	\begin{equation}
		(\beta_{t1}^\ell,\ldots,\beta_{t,t-1}^\ell)=[\boldsymbol{M}_{t+1}^{\beta,\ell}]_{t+1,2:t},
	\end{equation}
	where $[\boldsymbol{M}_{t+1}^{\alpha,\ell}]_{t+1,2:t+1}\in\mathbb{R}^t$ is a shorthand for $([\boldsymbol{M}_{t+1}^{\alpha,\ell}]_{t+1,2},\ldots,$ $([\boldsymbol{M}_{t+1}^{\alpha,\ell}]_{t+1,t+1})$ and $[\boldsymbol{M}_{t+1}^{\beta,\ell}]_{t+1,2:t}\in\mathbb{R}^{t-1}$ is defined analogously.
	
	\section{State evolution of ML-RI-GAMP}\label{sec:SE}
	\subsection{Model assumptions}\label{subsec:model}
	The empirical distribution of a vector $\bx\in \mathbb R^{n}$ is given by $ \frac{1}{n}\sum_{i=1}^{n} \delta_{x_i}$, where $\delta_{x_i}$ denotes a Dirac delta mass on $x_i$. Similarly, the joint empirical  distribution of the rows of a matrix $(\bx_1, \ldots, \bx_k) \in \mathbb R^{n \times k}$ is $\frac{1}{n} \sum_{i=1}^{n} \delta_{(x_{i, 1}, \ldots, x_{i, k})}$. To formally state our results, we define pseudo-Lipschitz functions and convergence in Wasserstein-$2$ distance. A function $\psi$: $\mathbb{R}^k\to\mathbb{R}$ is said to be pseudo-Lipschitz of order 2, if there exists $L>0$ s.t., for all $\boldsymbol{u},\boldsymbol{v}\in\mathbb{R}^k$, $|\psi(\boldsymbol{u})-\psi(\boldsymbol{v})|\leq L||\boldsymbol{u}-\boldsymbol{v}||(1+||\boldsymbol{u}||+||\boldsymbol{v}||).$
	We say that $(\bx_1, \ldots, \bx_k)\in\mathbb{R}^{n\times t}$ converges to $(X_1,\ldots,X_k)$ in Wasserstein-2 distance, denoted as $(\boldsymbol{x}_1,\ldots,\boldsymbol{x}_k)\overset{W_2}{\to}(X_1,\ldots,X_k)$, if for all pseudo-Lipschitz function $\psi$ of order 2, almost surely,
	\begin{equation}
		\lim_{n\to\infty}\frac{1}{n}\sum_{i=1}^n\psi(x_{i,1},,\ldots,x_{i,k})=\mathbb{E}[\psi(X_1,\ldots,X_k)].
	\end{equation}
	%	almost surely holds.
	
	In the \emph{multi-layer model} \eqref{model}, we let the design matrices $\{\boldsymbol{A}_\ell\}_{\ell=1}^L$ be mutually independent rotationally invariant matrices, that is, $\boldsymbol{A}_\ell=\boldsymbol{O}_\ell^T\boldsymbol{\Lambda}_\ell\boldsymbol{Q}_\ell$, where $\boldsymbol{\Lambda}_\ell=\text{diag}(\boldsymbol{\lambda}_\ell)\in\mathbb{R}^{n_{\ell+1}\times n_\ell}$ contains the singular values $\boldsymbol{\lambda}_\ell\in\mathbb{R}^{\min\{n_\ell,n_{\ell+1}\}}$ and $\boldsymbol{O}_\ell\in\mathbb{R}^{n_{\ell+1}\times n_{\ell+1}}$, $\boldsymbol{Q}_\ell\in\mathbb{R}^{n_\ell\times n_\ell}$ are independent Haar-distributed orthogonal matrices. We assume that, for $\ell\in [L]$, $\boldsymbol{\epsilon}^{\ell+1}\overset{W_2}{\to}\epsilon^{\ell+1}$, $\boldsymbol{x}^1\overset{W_2}{\to}X^1$ with finite second moments and $\boldsymbol{\lambda}_\ell\overset{W_2}{\to}\Lambda_\ell$ with compact support. We consider the large system limit in which, as $ n_{\ell}  \to \infty$, the fraction $\frac{n_{\ell+1}}{n_\ell}$ approaches a constant $\delta_\ell\in (0, \infty)$. For $\delta_\ell>1$, let $\tilde{\Lambda}_\ell$ be a mixture of $\Lambda_\ell$ (w.p. $1/\delta_\ell$) and a point mass at $0$ (w.p. $1-1/\delta_\ell$); for $\delta_\ell\le 1$, we set $\tilde{\Lambda}_\ell=\Lambda_\ell$. Then, the assumptions above imply that, as $n_\ell\to\infty$, $\kappa_{2k}^\ell\to \bar{\kappa}_{2k}^\ell$ almost surely, where $\{\bar{\kappa}_{2k}^\ell\}_{k\ge 1}$ are the rectangular free cumulants of $\tilde{\Lambda}_\ell$.
	
	For the \emph{denoising functions}, we assume that $f_t^\ell$ and $h_t^\ell$ are Lipschitz in each of their
	arguments. The partial derivatives $\partial_{x_k^\ell}f_t^\ell$, $\partial_{r_k^\ell}h_t$ and $\partial_{g^\ell} h_t^\ell$ are all continuous on sets of probability 1, under the laws of $(X_1^\ell,\ldots,X_t^\ell)$ and $(G^\ell, R_1^\ell, . . . , R_{t-1}^\ell)$ given by the state evolution in Algorithm 2. We note that these assumptions are no stronger than those required in \cite{pandit2020inference1} for the state evolution of ML-VAMP to hold.

	\subsection{Main result}
	The state evolution of ML-RI-GAMP is described in Algorithm 2. Our main result (Theorem \ref{thm:main}) shows that, for $\ell\in [L]$, $(\boldsymbol{x}_1^\ell-\bar{\mu}_1^\ell\boldsymbol{x}^\ell,\ldots,\boldsymbol{x}_t^\ell-\bar{\mu}_t^\ell\boldsymbol{x}^\ell)$ and $(\boldsymbol{g}^\ell,\boldsymbol{r}_1^\ell,\ldots,\boldsymbol{r}_t^\ell)$ converge in the large system limit to independent jointly Gaussian distributions with covariance matrices $\boldsymbol{\bar{\Sigma}}_t^\ell$ and $\boldsymbol{\bar{\Omega}}_t^\ell$, respectively. To track the mean vector $\bar{\boldsymbol{\mu}}_t^\ell\in\mathbb{R}^t$ and the covariance matrices $\boldsymbol{\bar{\Sigma}}_t^\ell, \boldsymbol{\bar{\Omega}}_t^\ell\in\mathbb{R}^{t\times t}$, we need to define auxiliary matrices $\boldsymbol{\bar{\Psi}}_{t+1}^\ell,\boldsymbol{\bar{\Phi}}_{t+1}^\ell,\boldsymbol{\bar{\Gamma}}_{t+1}^\ell,\boldsymbol{\bar{\Delta}}_{t+1}^\ell\in\mathbb{R}^{(t+1)\times (t+1)}$ for $\ell \in [L]$ as

	\begin{equation}
		\boldsymbol{\bar{\Psi}}_{t+1}^\ell=\left(\begin{matrix}
			0 & 0  & \cdots   & 0 & 0 \\
			0 & \mathbb{E}[\partial_1\hat{X}_1^\ell]  & 0   & \cdots & 0\\
			0 & \mathbb{E}[\partial_1\hat{X}_2^\ell]  & \mathbb{E}[\partial_2\hat{X}_2^\ell]   & \cdots & 0\\
			\vdots & \vdots  & \vdots  & \ddots   & \vdots & \\
			0 & \mathbb{E}[\partial_1\hat{X}_t^\ell] &\mathbb{E}[\partial_2\hat{X}_t^\ell]& \cdots\  & \mathbb{E}[\partial_t\hat{X}_t^\ell]
		\end{matrix}\right),
	\end{equation}
	\begin{equation}\label{eq:Phi}
		\boldsymbol{\bar{\Phi}}_{t+1}^\ell=\left(\begin{matrix}
			0 & 0  & \cdots   & 0 & 0 \\
			\mathbb{E}[\partial_gS_1^\ell] & 0  & 0   & \cdots & 0\\
			\mathbb{E}[\partial_gS_2^\ell] & \mathbb{E}[\partial_1S_2^\ell]  & 0   & \cdots & 0\\
			\vdots & \vdots  & \vdots  & \ddots   & \vdots & \\
			\mathbb{E}[\partial_gS_t^\ell] & \mathbb{E}[\partial_1S_t^\ell] & \cdots\  & \mathbb{E}[\partial_{t-1}S_t^\ell] & 0\\
		\end{matrix}\right),
	\end{equation}
	\begin{equation}
		\boldsymbol{\bar{\Gamma}}_{t+1}^\ell=\left(\begin{matrix}
			\mathbb{E}[(X^\ell)^2] & \mathbb{E}[X^\ell\hat{X}_1^\ell]   & \cdots & \mathbb{E}[X^\ell\hat{X}_t^\ell] \\
			\mathbb{E}[X^\ell\hat{X}_1^\ell] & \mathbb{E}[(\hat{X}_1^\ell)^2]    & \cdots & \mathbb{E}[\hat{X}_1^\ell \hat{X}_t^\ell]\\
			\mathbb{E}[X^\ell\hat{X}_2^\ell] & \mathbb{E}[\hat{X}_1^\ell \hat{X}_2^\ell]   & \cdots & \mathbb{E}[\hat{X}_2^\ell \hat{X}_t^\ell]\\
			\vdots & \vdots & \ddots   & \vdots & \\
			\mathbb{E}[X^\ell\hat{X}_t^\ell] & \mathbb{E}[\hat{X}_1^\ell \hat{X}_t^\ell] & \cdots\  & \mathbb{E}[(\hat{X}_t^\ell)^2]
		\end{matrix}\right),
	\end{equation}
	\begin{equation}\label{eq:Delta}
		\boldsymbol{\bar{\Delta}}_{t+1}^\ell=\left(\begin{matrix}
			0 & 0  & 0   & \cdots & 0 \\
			0 & \mathbb{E}[(S_1^\ell)^2]  & \mathbb{E}[S_1^\ell S_2^\ell]   & \cdots & \mathbb{E}[S_1^\ell S_t^\ell]\\
			0 & \mathbb{E}[S_1^\ell S_2^\ell] & \mathbb{E}[(S_2^\ell)^2]   & \cdots & \mathbb{E}[S_2^\ell S_t^\ell]\\
			\vdots & \vdots  & \vdots  & \ddots   & \vdots & \\
			0 & \mathbb{E}[S_1^\ell S_t^\ell] &\mathbb{E}[S_2^\ell S_t^\ell]& \cdots\  & \mathbb{E}[(S_t^\ell)^2]
		\end{matrix}\right).
	\end{equation}
	Here, for $\ell\in \{2, \ldots, L\}$, $X^\ell=q^\ell(G^{\ell-1},\epsilon^\ell)$, $\{G^{\ell}\}_{\ell\in [L]}$ are defined in Line 1 of Algorithm \ref{algo:SE}, $\{\hat{X}_k^\ell\}_{k\in [t], \ell\in [L]}$ are defined in Lines 4 and 8, and $\{S_k^\ell\}_{k\in \{2, \ldots, t\}, \ell\in [L]}$ are defined in Lines 12 and 14. The initial condition is calculated recursively from $\ell=L$ to $\ell=1$ as $S_1^L=h_1^L(q(G^L,\epsilon^{L+1}))$, $S_1^\ell=X_1^{\ell+1}$, where $X_1^{\ell+1}=\bar{\boldsymbol{\mu}}_1^{\ell+1}q(G^\ell,\epsilon^{\ell+1})+W_1^\ell$,  $\bar{\boldsymbol{\mu}}_1^\ell=\delta_\ell\bar{\kappa}_2^\ell\mathbb{E}[\partial_gh_1^\ell(X_1^{\ell+1})]$ and $\bar{\boldsymbol{\Omega}}_1^\ell=\delta_\ell\bar{\kappa}_2^\ell\mathbb{E}[h_1^\ell(X_1^{\ell+1})^2]+$ $\delta_\ell\bar{\kappa}_4\mathbb{E}[X_\ell^2](\mathbb{E}[\partial_{g^\ell}h_1^\ell(X_1^{\ell+1})])$. The partial derivatives are calculated in a similar way as in Section \ref{sec:AMP}.	
	
	Next, we compute the covariance matrix $\boldsymbol{\bar{\Sigma}}_{t+1}^{\ell}\in\mathbb{R}^{(t+1)\times (t+1)}$ from
	\begin{equation}\label{eq:compSigma}
		\boldsymbol{\bar{\Sigma}}_{t+1}^\ell=\sum_{j=0}^{2t+1}\bar{\kappa}_{2(j+1)}^\ell\boldsymbol{\Xi}_{t+1}^{(j),\ell},
	\end{equation}
	where, for $j\geq1$, $\boldsymbol{\Xi}_{t+1}^{(j),\ell}\in\mathbb{R}^{(t+1)\times (t+1)}$ is given by 
	\begin{equation}
		\begin{aligned}
			&\boldsymbol{\Xi}_{t+1}^{(j),\ell}=\sum_{i=1}^j(\boldsymbol{\bar{\Psi}}^\ell_{t+1}\boldsymbol{\bar{\Phi}}^\ell_{t+1})^i\boldsymbol{\bar{\Gamma}}^\ell_{t+1}((\boldsymbol{\bar{\Psi}}^\ell_{t+1}\boldsymbol{\bar{\Phi}}^\ell_{t+1})^T)^{j-i}\\
			&+\sum_{i=1}^{j-1}(\boldsymbol{\bar{\Psi}}^\ell_{t+1}\boldsymbol{\bar{\Phi}}^\ell_{t+1})^i\boldsymbol{\bar{\Psi}}^\ell_{t+1}\boldsymbol{\bar{\Delta}}^\ell_{t+1}(\boldsymbol{\bar{\Psi}}^\ell_{t+1})^T((\boldsymbol{\bar{\Psi}}^\ell_{t+1}\boldsymbol{\bar{\Phi}}^\ell_{t+1})^T)^{j-i-1},
		\end{aligned}
	\end{equation}
	with $\boldsymbol{\Xi}_{t+1}^{(0),\ell}=\boldsymbol{\bar{\Gamma}}_{t+1}^\ell$ and $\boldsymbol{\bar{\Sigma}}_1^\ell=\bar{\kappa}_2^\ell\mathbb{E}[(X^\ell)^2]$. Given  $\boldsymbol{\bar{\Sigma}}_{t+1}^{\ell}$, one obtains the random variable $S_{t+1}^\ell$ according to Line 12 in Algorithm \ref{algo:SE}, which in turn allows to compute $\boldsymbol{\bar{\Phi}}^\ell_{t+2},\boldsymbol{\bar{\Delta}}^\ell_{t+2}$ as in \eqref{eq:Phi} and \eqref{eq:Delta}, respectively. 
	
	\begin{algorithm}[tb]
		\caption{State evolution of ML-RI-GAMP}
		\begin{algorithmic}[1]\label{algo:SE}
			\STATE initialize: %$S_1^\ell$ according to the initialization of ML-RI-GAMP,
			$G^1\sim\mathcal{N}(0,\bar\kappa_2^1(X^1)^2)$, $G^\ell\sim\mathcal{N}(0,\bar\kappa_2^\ell\mathbb{E}[q^\ell(G^{\ell-1},\epsilon^\ell)^2])$ for $\ell\in \{2, \ldots, L\}$, and $Y=q^{L+1}(G^L,\epsilon^{L+1})$
			\FOR {$t=1,\ldots,T$}
			\STATE$(X_1^1,\ldots,X_t^1)=\bar{\boldsymbol{\mu}}_t^1 X^1+(W_1^1,\ldots,W_t^1)$, where $(W_1^1,\ldots,W_t^1)\sim\mathcal{N}(0,\boldsymbol{\bar{\Omega}}_t^1)$
			\STATE$\hat{X}_t^1=f_t^1(X_1^1,\ldots,X_t^1)$
			\STATE$(G^1,R_1^1,\ldots,R_t^1)\sim\mathcal{N}(0,\boldsymbol{\bar{\Sigma}}_{t+1}^1)$
			\FOR {$\ell=2:L$}
			\STATE$(X_1^\ell,\ldots,X_t^\ell)=\bar{\boldsymbol{\mu}}_t^\ell X^\ell+(W_1^\ell,\ldots,W_t^\ell)$, where $(W_1^\ell,\ldots,W_t^\ell)\sim\mathcal{N}(0,\boldsymbol{\bar{\Omega}}_t^\ell)$
			\STATE$\hat{X}_t^\ell=f_t^\ell(X_1^\ell,\ldots,X_t^\ell,R_1^{\ell-1},\ldots,R_{t-1}^{\ell-1})$
			\STATE$(G^\ell,R_1^\ell,\ldots,R_t^\ell)\sim\mathcal{N}(0,\boldsymbol{\bar{\Sigma}}_{t+1}^\ell)$
			\ENDFOR
			\FOR {$\ell=1:L-1$}
			\STATE$S_{t+1}^\ell=h_t^\ell(R_1^\ell,\ldots,R_t^\ell,X_1^{\ell+1},\ldots,X_t^{\ell+1})$
			\ENDFOR
			\STATE$S_{t+1}^L=h_t^L(R_1^L,\ldots,R_t^L,Y)$
			\ENDFOR
		\end{algorithmic}
	\end{algorithm}
	
	At this point, we compute $\boldsymbol{\bar{\Omega}}_{t+1}^\ell\in\mathbb{R}^{(t+1)\times (t+1)}$ from $\boldsymbol{\Omega}_{t+2}^{\prime\ell}\in\mathbb{R}^{(t+2)\times (t+2)}$ as
	\begin{equation}
		\boldsymbol{\bar{\Omega}}^\ell_{t+1}=[\boldsymbol{\Omega}_{t+2}^{\prime\ell}]_{2:t+2,2:t+2},
	\end{equation}
	where 
	\begin{equation}
		\boldsymbol{\Omega}_{t+2}^{\prime\ell}=\delta_\ell\sum_{j=0}^{2(t+1)}\bar{\kappa}_{2(j+1)}^\ell\boldsymbol{\Theta}_{t+2}^{(j),\ell},
	\end{equation}
	and for $j\geq1$, $\boldsymbol{\Theta}_{t+2}^{(j),\ell}\in\mathbb{R}^{(t+2)\times (t+2)}$ is
	\begin{equation}\label{eq:Theta_t1j_def}
		\begin{aligned}
			&\boldsymbol{\Theta}_{t+2}^{(j),\ell}=\sum_{i=1}^j(\boldsymbol{\bar{\Phi}}^\ell_{t+2}\boldsymbol{\bar{\Psi}}^\ell_{t+2})^i\boldsymbol{\bar{\Delta}}^\ell_{t+2}((\boldsymbol{\bar{\Phi}}^\ell_{t+2}\boldsymbol{\bar{\Psi}}^\ell_{t+2})^T)^{j-i}\\
			&+\sum_{i=1}^{j-1}(\boldsymbol{\bar{\Phi}}^\ell_{t+2}\boldsymbol{\bar{\Psi}}^\ell_{t+2})^i\boldsymbol{\bar{\Phi}}^\ell_{t+2}\boldsymbol{\bar{\Gamma}}^\ell_{t+2}(\boldsymbol{\bar{\Phi}}^\ell_{t+2})^T((\boldsymbol{\bar{\Phi}}^\ell_{t+2}\boldsymbol{\bar{\Psi}}^\ell_{t+2})^T)^{j-i-1},
		\end{aligned}
	\end{equation}
	with $\boldsymbol{\Theta}_{t+2}^{(0),\ell}=\boldsymbol{\bar{\Delta}}^\ell_{t+2}$ and $\boldsymbol{\bar{\Omega}}^\ell_1=\delta_\ell\bar{\kappa}_2^\ell\mathbb{E}[(S_1^\ell)^2]+\delta_\ell\bar{\kappa}_4^\ell\mathbb{E}[(X^\ell)^2](\mathbb{E}[\partial_gS_1^\ell])^2$.
	
	Finally, we evaluate the mean vector
	$\bar{\boldsymbol{\mu}}_{t+1}^\ell=(\bar{\mu}_1^\ell,\ldots,\bar{\mu}_{t+1}^\ell)$ recursively with $\bar{\mu}_{t+1}^\ell=[\boldsymbol{\bar{M}}_{t+2}^{\beta,\ell}]_{t+2,1}$, where
	\begin{equation}\label{eq:barM}
		\boldsymbol{\boldsymbol{\bar{M}}}_{t+2}^{\beta,\ell}=\delta_\ell\sum_{j=0}^{t+2}\bar{\kappa}_{2(j+1)}^\ell\boldsymbol{\bar{\Phi}}^\ell_{t+2}(\boldsymbol{\bar{\Psi}}^\ell_{t+2}\boldsymbol{\bar{\Phi}}^\ell_{t+2})^j.
	\end{equation}
	
	We note that the formulas for $\boldsymbol{\Theta}_{t+2}^{(j),\ell}$ and $\boldsymbol{\boldsymbol{\bar{M}}}_{t+2}^{\beta,\ell}$ in \eqref{eq:Theta_t1j_def} and \eqref{eq:barM}, respectively, involve also the matrices 
	$\boldsymbol{\bar{\Psi}}^\ell_{t+2}$ and $\boldsymbol{\bar{\Gamma}}^\ell_{t+2}$, which have not been computed just yet. However, the last rows and columns of these matrices do not influence the calculation, because of the form of $\boldsymbol{\bar{\Delta}}^\ell_{t+2}$ and  $\boldsymbol{\bar{\Phi}}^\ell_{t+2}$. Therefore, the expressions depend only on the top left submatrices of $\boldsymbol{\bar{\Psi}}^\ell_{t+2}$ and $\boldsymbol{\bar{\Gamma}}^\ell_{t+2}$, which are equal to $\boldsymbol{\bar{\Psi}}^\ell_{t+1}$ and $\boldsymbol{\bar{\Gamma}}^\ell_{t+1}$ (and, hence, have been already computed).  We also remark that the matrices $\boldsymbol{\bar{\Omega}}^\ell_{t}$ and $\boldsymbol{\bar{\Sigma}}^\ell_{t}$ are the top left submatrices of $\boldsymbol{\bar{\Omega}}^\ell_{t+1}$ and $\boldsymbol{\bar{\Sigma}}^\ell_{t+1}$, respectively. %Similarly, the mean vector $\barbmu_{t+1}$ is obtained by appending $\bar{\mu}_{t+1}$ to $\barbmu_{t}$.
	
	%	\marco{I suggest to move the two equations above to the place in which that derivative appears for the first time, namely, after the definition of $\bar{\boldsymbol \Phi}_{t+1}^\ell$} \yizhou{I move this part to the end of the implementation section.}

	At this point, we are ready to state our main result. 
	
	\newtheorem{theorem}{Theorem}
	\begin{theorem}\label{thm:main}
		Consider the multi-layer model in \eqref{model} with the assumptions in Section \ref{subsec:model}, the ML-RI-GAMP in Algorithm \ref{algo:MLRIGAMP} and its state evolution in Algorithm \ref{algo:SE}. Let $\psi$: $\mathbb{R}^{2t+1}\to\mathbb{R}$ and $\phi$: $\mathbb{R}^{2t+2}\to\mathbb{R}$ be any pseudo-Lipschitz functions of order 2. Then, for each $t\geq1$ and $\ell\in [L]$, we almost surely
		have
		\begin{equation}
			\begin{aligned}
				\lim_{n_\ell\to\infty}\frac{1}{n_\ell}\sum_{i=1}^{n_\ell}&\psi(x_{1,i}^\ell,\ldots,x_{t,i}^\ell,\hat{x}_{1,i}^\ell,\ldots,\hat{x}_{t,i}^\ell,x_i^\ell)\\&=\mathbb E\left [\psi(X_1^\ell,\ldots,X_t^\ell,\hat{X}_1^\ell,\ldots,\hat{X}_t^\ell,X^\ell)\right],
			\end{aligned}
		\end{equation}
		\begin{equation}
			\begin{aligned}
				\lim_{n_{\ell+1}\to\infty}\frac{1}{n_{\ell+1}}\sum_{i=1}^{n_{\ell+1}}&\phi(r_{1,i}^\ell,\ldots,r_{t,i}^\ell,s_{1,i}^\ell,\ldots,s_{t+1,i}^\ell,g_i^\ell)\\&=\mathbb E\left[\phi(R_1^\ell,\ldots,R_t^\ell,S_1^\ell,\ldots,S_{t+1}^\ell,G^\ell)\right].
			\end{aligned}
		\end{equation}
		%		\marco{I think that in the second equation, you want $G$ and not $X$ as the last argument. I also fixed a couple of typos, pls check} \yizhou{Yes, it should be $G$.}
	\end{theorem}
	
	The result above readily allows to evaluate the overlaps and mean squared errors between ML-RI-GAMP iterates and feature vectors. In particular, by choosing suitably the pseudo-Lipschitz function $\psi$, we have that, for $\ell\in [L]$, $\frac{1}{n_\ell}\|\hat{\bx}_t^\ell-\bx^\ell\|^2_2$  $ \to \mathbb E\{(\hat{X}_t^\ell-X^\ell)^2\}$ and $|\langle\hat{\bx}_t^\ell,\bx^\ell\rangle|^2/(\|\hat{\bx}^\ell_t\|^2\|\bx^\ell\|^2)\to(\mathbb E[\hat{X}_t^\ell \, X^\ell])^2/(\mathbb E[(\hat{X}_t^\ell)^2]\mathbb E[(X^\ell)^2])$.
	
	For the \emph{initialization}, we remark that it is effective only if $\mathbb{E}[X_1^\ell X^\ell]>0$. Otherwise, the algorithm will remain stuck at a trivial fixed point with zero correlation to the ground truth. For most nonlinear observation, including 1-bit compressed sensing and linear regression, this condition holds true. For other situations such as phase retrieval, an informative initialization $\boldsymbol{x}_1^\ell\overset{W_2}{\to}X_1^\ell$ with finite second moments is more advisable.

	\subsection{Proof sketch}
	The full proof is deferred to the appendix. The idea is to regard the ML-RI-GAMP algorithm as an instance of the general multi-layer recursion presented in Appendix \ref{app:gen}. The state evolution of this general recursion is then proved in Appendix \ref{app:thmgen} via the conditioning technique introduced in \cite{fan2022approximate} for low-rank matrix estimation and then extended to the model \eqref{model} in the special case $L=1$ in \cite{venkataramanan2022estimation}.
	Finally, Appendix \ref{app:pfmain} contains the reduction from the general recursion to our proposed ML-RI-GAMP algorithm. This involves a delicate induction argument, together with a careful choice of side information, initialization and denoisers of the general recursion, and it concludes the proof of Theorem \ref{thm:main}.
	
	\section{Numerical results}\label{sec:impl}
	We consider a neural network with ReLU activation:	
	\begin{equation}
		\boldsymbol{x}^{\ell}=\text{ReLU}(\boldsymbol{g}^{\ell-1}),\ \boldsymbol{g}^{\ell-1}=\boldsymbol{A}_{\ell-1}\boldsymbol{x}^{\ell-1},\,\, \mbox{for}\ \ell\in \{2, \ldots, L\},
	\end{equation}
	and $\boldsymbol{x}^{L+1}=\boldsymbol{A}_L\boldsymbol{g}^L+\boldsymbol{\epsilon}^{L+1}$, where $\epsilon^{L+1}\sim\mathcal{N}(0,\sigma^2)$ and $\text{ReLU}(x)=\max(x, 0)$. We consider the Gaussian prior $X^1\sim\mathcal{N}(0,1)$ and two (rotationally invariant) choices for the distribution of the weight matrices: \emph{(i)} i.i.d. Gaussian elements (leading to a Marchenko-Pastur spectrum), and \emph{(ii)} eigenvalues sampled i.i.d. from a $\sqrt{6}$ Beta(1, 2) distribution (the normalization of the Beta(1, 2) distribution is chosen to ensure a unit second moment). For the former choice, ML-RI-GAMP reduces to the ML-AMP algorithm in \cite{manoel2017multi}. For the initialization, we pick $h_1^\ell(y)=y$, which allows to escape the trivial fixed point of the state evolution. The other denoisers are chosen to be the following posterior means:
	\begin{equation}
		\begin{split}
			&f_t^1(x_1^1,\ldots,x_t^1)=\mathbb{E}[X^\ell|X_1^1=x_1^1,\ldots,X_t^1=x_t^1],\\
			&	f_t^\ell(x_1^\ell,...,x_t^\ell,r_1^{\ell-1},...,r_{t-1}^{\ell-1})=\mathbb{E}[X^\ell|X_1^\ell=x_1^\ell,...,X_t^\ell=x_t^\ell,\\
			&\quad R_1^{\ell-1}=r_1^{\ell-1},\ldots,R_{t-1}^{\ell-1}=r_{t-1}^{\ell-1}], \quad \ell\in \{2, \ldots, L\},
		\end{split}
		\label{Bayesian estimation1}
	\end{equation}
	and, for $\ell\in [L-1]$,
	\begin{equation}
		\begin{aligned}
			&h_{t+1}^\ell(r_1^\ell,...,r_t^\ell,x_1^{\ell+1},...,x_t^{\ell+1})=\mathbb{E}[G^\ell|R_1^\ell=r_1^\ell,...,R_t^\ell=r_t^\ell,\\&X_1^{\ell+1}=x_1^{\ell+1},...,X_t^{\ell+1}=x_t^{\ell+1}]-\mathbb{E}[G^\ell|R_1^\ell=r_1^\ell,...,R_t^\ell=r_t^\ell],\\
			&h_{t+1}^L(r_1^L,\ldots,r_t^L,y)=\mathbb{E}[G^L|R_1^L=r_1^L,\ldots,R_t^L=r_t^L,Y=y]\\&\quad-\mathbb{E}[G^L|R_1^L=r_1^L,\ldots,R_t^L=r_t^L].
		\end{aligned}
		\label{Bayesian estimation2}
	\end{equation} 
	The details of the implementation -- together with additional numerical results -- can be found in Appendix \ref{app:implement}.
	
	\begin{figure}[tb]
		\centering	
		\includegraphics[width=\linewidth]{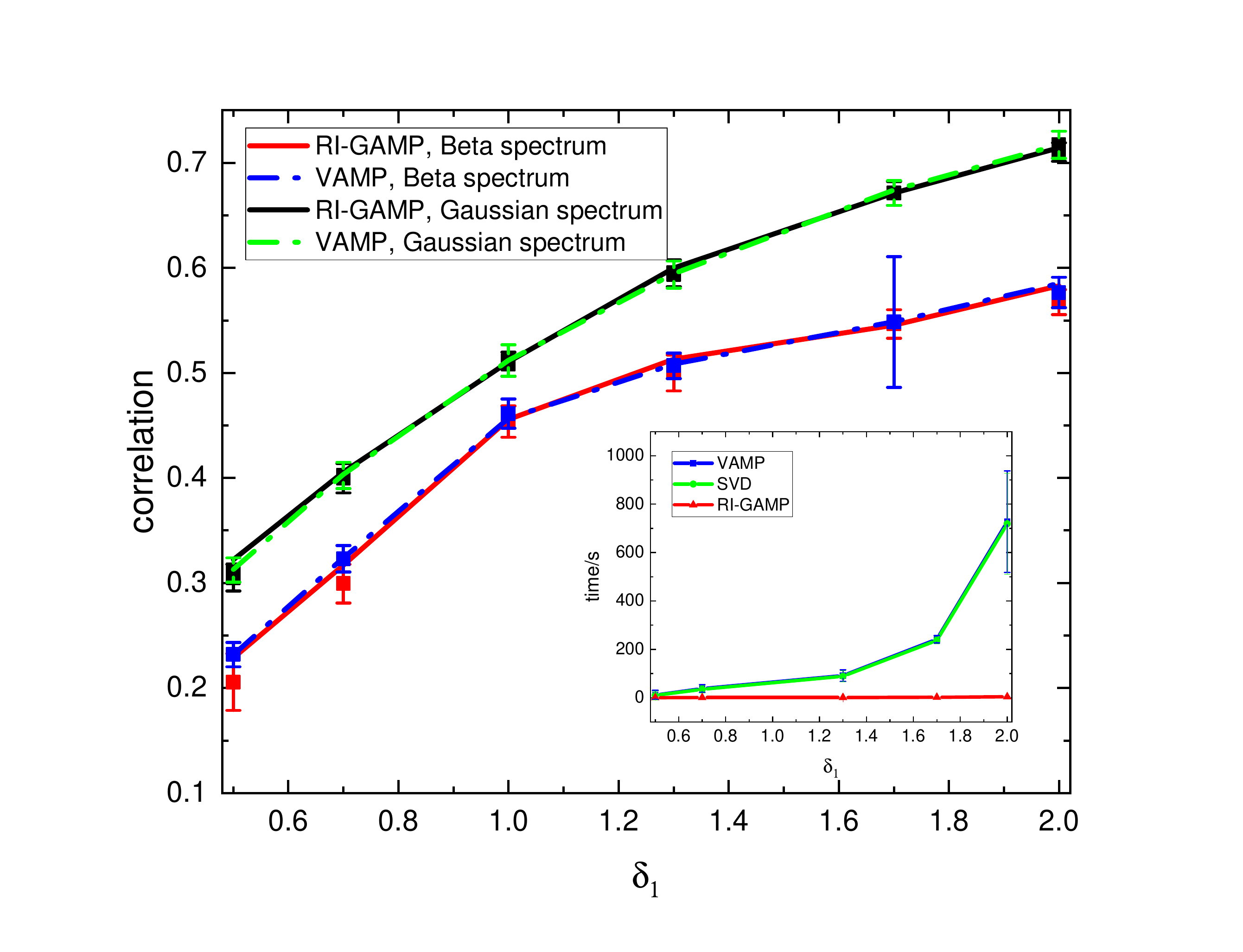}
		\caption{ML-RI-GAMP achieves a performance close to ML-VAMP, but with a much lower computational complexity. $L=2$, $\sigma=0.2$, $n_1=3000$ and $\delta_2=1.3$. % for various choices of $\delta_1$ and spectra. The running time of ML-RI-GAMP, ML-VAMP and SVD within 10 iterations for various choices of $\delta_1$ are shown on the bottom right.
			\label{fig:fig2}}
	\end{figure}
	
	In our numerical experiments, we perform 100 independent trials and report the mean and the standard error of the overlap $\frac{|\langle\hat{\boldsymbol{x}}_t^1,\boldsymbol{x}^1\rangle|^2}{||\hat{\boldsymbol{x}}_t^1||^2||\boldsymbol{x}^1||^2}$ between the signal $\boldsymbol{x}^1$ and the iterate $\hat{\boldsymbol{x}}^1_t$. \footnote{Our code is shared in https://github.com/sparc-lab/ML-RI-GAMP.} In Fig. \ref{fig:fig2}, we compare the performance of ML-RI-GAMP and ML-VAMP \cite{pandit2020inference1,gabrie2018entropy} as a function of $\delta_1$. As the fixed point of ML-VAMP was proved to be optimal\cite{pandit2020inference1}, our results suggest that ML-RI-GAMP achieves Bayes-optimal performance as well. Furthermore, on the bottom right side of the figure, we show that the complexity of ML-VAMP is dominated by the SVD, which requires $O(n_1^3)$ time. In contrast, our ML-RI-GAMP algorithm runs in $O(n_1^2)$ time. In Fig. \ref{fig:fig2}, ML-RI-GAMP is 30 times faster than ML-VAMP for $\delta_1=0.7$, and nearly 200 times faster for $\delta_1=2.0$.
	
	\section{Concluding remarks}
	In this paper, we propose a novel approximate message passing algorithm, dubbed ML-RI-GAMP, for inference in multi-layer models with rotationally invariant weight matrices, and we rigorously characterize its performance in the high-dimensional limit via a state evolution recursion. Our proposed ML-RI-GAMP recovers the existing algorithm in the Gaussian setting, and this feature can potentially be used to capture the effect of mismatch in the noise statistics \cite{barbier2022price}. We also show that ML-RI-GAMP matches the performance of ML-VAMP, but it has a significantly lower complexity requirement. 
	
	\section*{Acknowledgements}
	We thank Jean Barbier for helpful discussions.
	%	\small
	\bibliographystyle{ieeetr}
	\bibliography{ref}
	
	\appendix 
	
	\subsection{Implementation details and additional numerical results}\label{app:implement}
	To explicitly calculate %the Bayesian estimation 
	\eqref{Bayesian estimation1} and \eqref{Bayesian estimation2} for middle layers with ReLU activation, we first introduce a joint distribution function
	\begin{equation}
		p(z_0,z_1)\propto f(z_1|z_0)\exp\left(-\frac{(z_0-r_0)^2}{2\sigma_0^2}-\frac{(z_1-r_1)^2}{2\sigma_1^2}\right),
	\end{equation}
	where $f(z_1|z_0)=\delta(z_1-z_0)\mathbf{1}_{z_0>0}+\delta(z_1)\mathbf{1}_{z_0\leq0}$ accounts for the ReLU activation. Then, the marginal distributions of $z_0$ and $z_1$ can be expressed as
	\begin{equation}
		p(z_0)\propto C_P\phi(z_0;r_P,\sigma_P^2)+C_N\phi(z_0;r_0,\sigma_0^2)\label{distributionz0},
	\end{equation}        
	\begin{equation}
		p(z_1)\propto C_P\phi(z_1;r_P,\sigma_P^2)+C_N\delta(z_1)\Phi\left(-\frac{r_0}{\sigma_0}\right)\label{distributionz1},
	\end{equation}           
	where $r_P=\frac{r_0\sigma_1^2+r_1\sigma_0^2}{\sigma_0^2+\sigma_1^2}$, $\sigma_P^2=\frac{\sigma_0^2\sigma_1^2}{\sigma_0^2+\sigma_1^2}$, $C_P=\exp(-\frac{r_0^2}{2\sigma_0^2}-\frac{r_1^2}{2\sigma_1^2}+\frac{r_P^2}{2\sigma_P^2})$, $C_N=\exp(-\frac{r_1^2}{2\sigma_1^2})$, $\phi(x;\mu,\sigma^2)=\frac{1}{\sqrt{2\pi}\sigma}e^{-\frac{(x-\mu)^2}{2\sigma^2}}$ denotes the probability density function of the normal distribution, and $\Phi(x)=\int_{-\infty}^x\phi(t;0,1)\,{\rm d}t$ the corresponding cumulative distribution function. The derivative with respect to $z_0$ can be calculated as
	\begin{equation}
		\mathbb{E}\left[\frac{\partial \mathbb{E}[z_0|r_0,r_1,\sigma_0^2,\sigma_1^2]}{\partial r_0}\right]=\frac{1}{\sigma_0^2}\text{Var}[z_0|r_0,r_1,\sigma_0^2,\sigma_1^2],
	\end{equation}  
	and the derivative with respect to $z_1$ is similar, where the posterior mean and variance is with respect to the marginal distribution \eqref{distributionz0} or \eqref{distributionz1}.
	
	Going back to ML-RI-GAMP, as $(X_1^\ell-\bar{\mu}^\ell_1X^\ell,\ldots,X_t^\ell-\bar{\mu}^\ell_tX^\ell)\sim\mathcal{N}(0,\boldsymbol{\bar{\Omega}}_t^\ell)$ and  $(G^\ell,R_1^\ell,\ldots,R_t^\ell)\sim\mathcal{N}(0,\boldsymbol{\bar{\Sigma}}_{t+1}^\ell)$, the conditional distributions of $X^\ell$ and $G^\ell$ are
	\begin{equation}
		\begin{split}
			X^\ell & \mid (X_1^\ell=x_1^\ell,\ldots,X_t^\ell=x_t^\ell)\sim\mathcal{N}(\tilde{x}_t^\ell,\tilde{\rho}_t^\ell),\\G^\ell & \mid (R_1^\ell=r_1^\ell,\ldots,R_t^\ell=r_t^\ell)\sim\mathcal{N}(\tilde{r}_t^\ell,\tilde{\sigma}_t^\ell),
		\end{split}
	\end{equation}
	where $\tilde{x}_t^\ell=\frac{(\boldsymbol{\bar{\mu}}_t^\ell)^T(\boldsymbol{\bar{\Omega}}_t^\ell)^{-1}\boldsymbol{\bar{x}}^\ell}{(\boldsymbol{\bar{\mu}}_t^\ell)^T(\boldsymbol{\bar{\Omega}}_t^\ell)^{-1}\boldsymbol{\bar{\mu}}_t^\ell}$, $\tilde{\rho}_t^\ell=\frac{1}{(\boldsymbol{\bar{\mu}}_t^\ell)^T(\boldsymbol{\bar{\Omega}}_t^\ell)^{-1}\boldsymbol{\bar{\mu}}_t^\ell}$, $\tilde{r}_t^\ell=(\boldsymbol{\bar{\Sigma}}_{t+1}^\ell)_{[1,2:t+1]}((\boldsymbol{\bar{\Sigma}}_{t+1}^\ell)_{[2:t+1,2:t+1]})^{-1}\boldsymbol{\bar{r}}_t^\ell$, $\tilde{\sigma}_t^\ell=\mathbb{E}[(G^\ell)^2]-(\boldsymbol{\bar{\Sigma}}_{t+1}^\ell)_{[1,2:t+1]}((\boldsymbol{\bar{\Sigma}}_{t+1}^\ell)_{[2:t+1,2:t+1]})^{-1}(\boldsymbol{\bar{\Sigma}}_{t+1}^\ell)_{[2:t+1,1]}$ and $\boldsymbol{\bar{x}}_t^\ell$ (resp.\ $\boldsymbol{\bar{r}}_t^\ell$) is a shorthand for $(x_1^\ell,\ldots,x_t^\ell)$ (resp.\ $(r_1^\ell,\ldots,r_t^\ell)$). Then, the posterior estimation becomes
	\begin{equation}
		f_t^\ell(x_1^\ell,\ldots,x_t^\ell,r_1^{\ell-1},\ldots,r_t^{\ell-1})=\mathbb{E}[z_1|\tilde{r}_t^{\ell-1},\tilde{x}_t^\ell,\tilde{\rho}_t^\ell,\tilde{\sigma}_t^{\ell-1}],
	\end{equation}
	\begin{equation}
		\begin{split}
			h_{t+1}(r_1^\ell,&\ldots,r_t^\ell,x_1^{\ell+1},\ldots,x_t^{\ell+1})\\
			&=\mathbb{E}[z_0|\tilde{r_t}^\ell,\tilde{x_t}^{\ell+1},\tilde{\rho_t}^{\ell+1},\tilde{\sigma_t}^\ell]-\tilde{r}_t^\ell,
		\end{split}
	\end{equation}
	and the derivatives are given by
	\begin{equation}
		\begin{aligned}
			\partial_{x_k^\ell}f_t^\ell & (x_0^\ell,\ldots,x_t^\ell,r_0^{\ell-1},r_t^{\ell-1})\\
			&=\frac{1}{\tilde{\rho}_t^\ell}\text{Var}[z_1|\tilde{r}_t^{\ell-1},\tilde{x}_t^\ell,\tilde{\rho}_t^\ell,\tilde{\sigma}_t^{\ell-1}]\cdot\frac{(\boldsymbol{\bar{\mu}}_t^\ell)^T(\boldsymbol{\bar{\Omega}}_t^\ell)^{-1}\boldsymbol{e}_k}{(\boldsymbol{\bar{\mu}}_t^\ell)^T(\boldsymbol{\bar{\Omega}}_t^\ell)^{-1}\boldsymbol{\bar{\mu}}_t^\ell},
		\end{aligned}
	\end{equation}
	\begin{equation}
		\begin{aligned}
			\partial_{r_k^\ell}h_{t+1} & (r_1^\ell,\ldots,r_t^\ell,x_1^{\ell+1},\ldots,x_t^{\ell+1})\\
			&=\bigg(\frac{1}{\tilde{\sigma}_t^\ell}\text{Var}[z_0|\tilde{r}_t^\ell,\tilde{x}_t^{\ell+1},\tilde{\rho}_t^{\ell+1},\tilde{\sigma}_t^\ell]-1\bigg)\\&\quad \cdot(\boldsymbol{\bar{\Sigma}}_{t+1}^\ell)_{[1,2:t+1]}((\boldsymbol{\bar{\Sigma}}_{t+1}^\ell)_{[2:t+1,2:t+1]})^{-1}\boldsymbol{e}_k,
		\end{aligned}
	\end{equation}
	where $\boldsymbol{e}_k\in\mathbb{R}^t$ represents the $k$-th element of the canonical basis.
	
	The denoisers $f_t^1$ and $h_{t+1}^L$ are the same as in \cite{venkataramanan2022estimation}. In particular, for the Gaussian prior, we have
	\begin{equation}
		f_t^1(x_1^1,\ldots,x_t^1)=\frac{(\boldsymbol{\bar{\mu}}_t^1)^T(\boldsymbol{\bar{\Omega}}_t^1)^{-1}\boldsymbol{\bar{x}}^1}{1+(\boldsymbol{\bar{\mu}}_t^\ell)^T(\boldsymbol{\bar{\Omega}}_t^1)^{-1}\boldsymbol{\bar{\mu}}_t^1},
	\end{equation}
	\begin{equation}
		\partial_{x_k^1}f_t^1(x_1^1,\ldots,x_t^1)=\frac{(\boldsymbol{\bar{\mu}}_t^\ell)^T(\boldsymbol{\bar{\Omega}}_t^1)^{-1}\boldsymbol{e}_k}{1+(\boldsymbol{\bar{\mu}}_t^1)^T(\boldsymbol{\bar{\Omega}}_t^\ell)^{-1}\boldsymbol{\bar{\mu}}_t^1};
	\end{equation}
	for the Rademacher prior, we have
	\begin{equation}
		f_t^1(x_1^1,\ldots,x_t^1)=\tanh((\boldsymbol{\bar{\mu}}_t^1)^T(\boldsymbol{\bar{\Omega}}_t^1)^{-1}\boldsymbol{\bar{x}}^1),
	\end{equation}
	\begin{equation}
		\begin{split}
			\partial_{x_k^1}f_t^1(x_1^1,\ldots,x_t^1)=(1-\tanh((\boldsymbol{\bar{\mu}}_t^1)^T(\boldsymbol{\bar{\Omega}}_t^1)^{-1}&\boldsymbol{\bar{x}}^1))^2\boldsymbol{\bar{\mu}}_t^1)^T\\
			&(\boldsymbol{\bar{\Omega}}_t^1)^{-1}\boldsymbol{e}_k.
		\end{split}
	\end{equation}
	Furthermore,
	%	for the Rademacher prior. The denoiser for the Gaussian noise of the output layer is
	\begin{equation}
		\begin{aligned}
			&h_{t+1}^L(r_1^L,\ldots,r_t^L,y)\\
			&=(\boldsymbol{\bar{S}}^L_{t+2})_{[1,2:t+2]}((\boldsymbol{\bar{S}}^L_{t+2})_{[2:t+2,2:t+2]})^{-1}\left[\begin{matrix}
				\boldsymbol{\bar{r}}_t\\\boldsymbol{y}
			\end{matrix}\right]-\boldsymbol{\bar{\Sigma}}_{t+1}^\ell)_{[1,2:t+1]}\\
			&\qquad\qquad\qquad\qquad\qquad\qquad\qquad ((\boldsymbol{\bar{\Sigma}}_{t+1}^\ell)_{[2:t+1,2:t+1]})^{-1}\boldsymbol{\bar{r}}_t,
		\end{aligned}
	\end{equation}
	\begin{equation}
		\begin{aligned}
			&\partial_{r_k^L}h_{t+1}(r_1^L,\ldots,r_t^L,y)\\
			&=(\boldsymbol{\bar{S}}^L_{t+2})_{[1,2:t+2]}((\boldsymbol{\bar{S}}^L_{t+2})_{[2:t+2,2:t+2]})^{-1}\left[\begin{matrix}
				\boldsymbol{e}_k\\0
			\end{matrix}\right]
			-\boldsymbol{\bar{\Sigma}}_{t+1}^\ell)_{[1,2:t+1]}\\
			&\qquad \qquad\qquad \qquad\qquad\qquad ((\boldsymbol{\bar{\Sigma}}_{t+1}^\ell)_{[2:t+1,2:t+1]})^{-1}\boldsymbol{e}_k,
		\end{aligned}
	\end{equation}
	where
	\begin{equation}
		\boldsymbol{\bar{S}}^L_{t+2}=\left[\begin{matrix}\boldsymbol{\bar{\Sigma}}^L_{t+1}&[\boldsymbol{\bar{\Sigma}}^L_{t+1}]_{[1:t+1,1]}\\
			[\boldsymbol{\bar{\Sigma}}^L_{t+1}]_{[1,1:t+1]}&\mathbb{E}[Y^2]\end{matrix}\right].
	\end{equation}
	Finally, 
	%as $\epsilon^{\ell+1}$ in \eqref{partial_gell} is not known, the partial derivative $\partial_{g^\ell}q^{\ell+1}(g^\ell,\epsilon^{\ell+1})$ can be estimated by $\mathbb{E}_{\epsilon^{\ell+1}}\partial_{g^\ell}q^{\ell+1}(g^\ell,\epsilon^{\ell+1})$. Given the state evolution, the partial derivative can be further estimated by
	%\begin{equation}
	%		\begin{aligned}
	%\partial_gh_{t+1}^\ell=&\frac{\partial}{\partial g^\ell}\mathbb{E}_{\{W^\ell_i\}_{i=1}^t,\epsilon^{\ell+1}}[h_{t+1}^\ell(r_1^\ell,\ldots,r_t^\ell,\bar{\mu}_1^{\ell+1}q^{\ell+1}(g^\ell,\epsilon^{\ell+1})+\\
	%&W_1^\ell,\ldots,\bar{\mu}_t^{\ell+1}q^{\ell+1}(g^\ell,\epsilon^{\ell+1})+W_t^\ell)].
	%\end{aligned}
	%\end{equation}
	an application of Stein's lemma gives that %(see (58) of \cite{venkataramanan2022estimation}), 
	%	its expectation can be calculated by
	\begin{equation}
		\mathbb{E}[\partial_gh_{t+1}^\ell]=\frac{1}{\tilde{\sigma}_t^2}\mathbb{E}[h_{t+1}^\ell(R_1^\ell,\ldots,R_t^\ell,X_1^{\ell+1},\ldots,X_t^{\ell+1})^2].
	\end{equation}
	%	We use this simpler form in the implementation.
	
	For the additional numerical results presented in Figure \ref{fig:fig1}, we consider two choices for the prior of the first layer: \emph{(i)} Gaussian prior $X^1\sim\mathcal{N}(0,1)$, and \emph{(ii)} Rademacher prior $\mathbb{P}(X^1=-1)=\mathbb{P}(X^1=-1)=1/2$. We also consider two (rotationally invariant) choices for the distribution of the weight matrices: \emph{(i)} i.i.d. Gaussian elements and \emph{(ii)} eigenvalues sampled i.i.d. from a $\sqrt{6}$ Beta(1, 2) distribution. For the former choice, we have $\bar{\kappa}_2^\ell=1$,  $\bar{\kappa}_{2k}^\ell=0$ for $k>1$ and, hence, our proposed ML-RI-GAMP reduces to the ML-AMP algorithm in \cite{manoel2017multi}. For the latter choice, the rectangular cumulants are obtained via \eqref{eq:fc} from the moments $\bar{m}_{2k}^\ell=\frac{6^k}{(k+1)(2k+1)}$ for $\delta_\ell<1$ and $\bar{m}_{2k}^\ell=\frac{1}{\delta_\ell}\frac{6^k}{(k+1)(2k+1)}$ for $\delta_\ell\geq1$. Non-linear estimation functions have been described before.
	%	We emphasize that the data priors and spectra only influence the non-linear estimation functions and free cumulants, while other parts remain the same. We also note that when $\{\boldsymbol{A_\ell}\}_{\ell=1}^L$ are iid Gaussian, $\{\kappa_{2k}^\ell\}_{k=2}^\infty$ become zero, and thus ML-RI-GAMP reduces to the ML-AMP algorithm in \cite{manoel2017multi}.
	We set $L=2$, $\sigma=0.2$ $n_1=3000$, $\delta_1=2$ and $\delta_2=1.3$. We perform 100 independent trials and report the mean and the standard error. In Fig. \ref{fig:fig1}, we plot the overlap between the signal $\boldsymbol{x}^1$ and the ML-RI-GAMP iterate $\hat{\boldsymbol{x}}^1_t$, as a function of the iteration number $t$. Different plots correspond to different priors (Gaussian, Rademacher) and different spectral distributions for the weight matrices (Gaussian, Beta). %shows the iteration curves  for Gaussian, Rademacher priors and Gaussian, Beta spectra using the normalized squared correlation $\frac{\langle\boldsymbol{\hat{x}}_t^\ell,\boldsymbol{x}^\ell\rangle}{||\hat{x}_t^\ell||^2||\boldsymbol{x}^\ell||^2}$. 
	In all the four settings taken into account, the state evolution predictions coming from Theorem \ref{thm:main} (solid lines) match well the empirical results represented by the markers. %Solid lines represent SE and the markers represent the empirical performances. The SE of ML-RI-GAMP in these situations is also plotted. It demonstrates that ML-RI-GAMP converges to its SE well for all cases with small variances. 
	
	\begin{figure}[tb]
		\centering	
		\includegraphics[width=\linewidth]{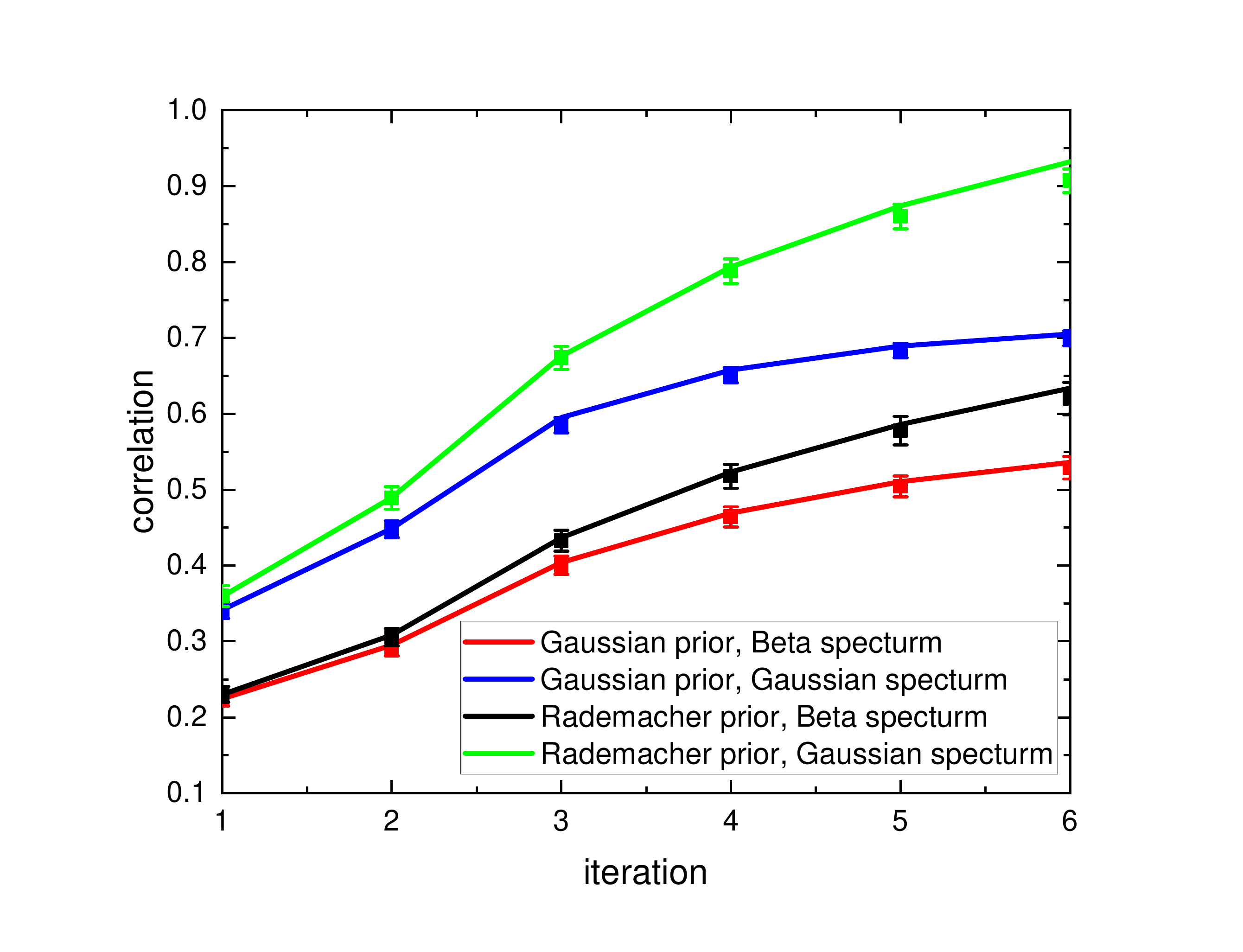}
		\caption{The empirical performance ML-RI-GAMP matches well the state evolution prediction of Theorem \ref{thm:main}. 
			\label{fig:fig1}}
	\end{figure}

	\subsection{General recursion}\label{app:gen}
	
	We propose a multi-layer general recursion described in Algorithm \ref{algo:gen}, with its state evolution in Algorithm \ref{algo:genSE}. Line 19, 23, 26 are the core parts with Onsager terms $\{\hat{\alpha}_{ti}^\ell\}_{\ell=1}^L$ and  $\{\hat{\beta}_{ti}^\ell\}_{\ell=1}^L$. We deal with the first step a little differently to match the initialization of ML-RI-GAMP. In the general recursion, $\boldsymbol{E}_\ell\in\mathbb{R}^{n_{\ell+1}\times k_{1\ell}}$ and $\boldsymbol{F}_\ell\in\mathbb{R}^{n_\ell\times k_{2\ell}}$ denote the side information that converges row-wise in $W_2$ to random vectors $E_\ell\in \mathbb R^{k_{1\ell}}$ and $F_\ell\in \mathbb R^{k_{2\ell}}$. The general recursion can also be non-rigorously derived with the generating functional analysis of \cite{ccakmak2020dynamical}.
	
	\begin{algorithm}[tb]  
		\caption{ML-RI-General Recursion}
		\begin{algorithmic}[1]\label{algo:gen}
			\STATE initialize: $\boldsymbol{u}_1^\ell$ i.i.d. (independent from $\boldsymbol{A}_\ell$, but not necessarily independent from one another)
			\STATE$\boldsymbol{z}_1^1=\boldsymbol{A}_1^T\boldsymbol{u}_1^1$
			\STATE$\boldsymbol{v}_1^1=\tilde{f}_1^1(\boldsymbol{z}_1^1,\boldsymbol{E}_1)$	
			\STATE$\boldsymbol{m}_1^1=\boldsymbol{A}_1\boldsymbol{v}_1^1-\hat{\alpha}_{t1}^1\boldsymbol{u}_1^1$
			\FOR {$\ell=2:L$}
			\STATE$\boldsymbol{z}_1^\ell=\boldsymbol{A}_\ell^T\boldsymbol{u}_1^\ell$
			\STATE$\boldsymbol{v}_1^\ell=\tilde{f}_1^\ell(\boldsymbol{z}_1^\ell,\boldsymbol{m}_1^{\ell-1},\boldsymbol{E}_\ell)$
			\STATE$\boldsymbol{m}_1^\ell=\boldsymbol{A}_\ell\boldsymbol{v}_1^\ell-\hat{\alpha}_{t1}^\ell\boldsymbol{u}_1^\ell$
			\ENDFOR
			\STATE$\boldsymbol{u}_2^L=\tilde{h}_2^L(\boldsymbol{m}_1^L,\boldsymbol{F}_L)$
			\STATE$\boldsymbol{z}_2^1=\boldsymbol{A}_1^T\boldsymbol{u}_2^1
			-\hat{\beta}_{2,1}^1\boldsymbol{v}_1^1$
			\FOR {$\ell=L-1:1$}
			\STATE$\boldsymbol{u}_2^\ell=\tilde{h}_2^\ell(\boldsymbol{m}_1^\ell,\boldsymbol{z}_1^{\ell+1},\boldsymbol{z}_2^{\ell+1},\boldsymbol{F}_\ell)$
			\STATE$\boldsymbol{z}_2^\ell=\boldsymbol{A}_\ell^T\boldsymbol{u}_2^\ell- \hat{\beta}_{2,1}^\ell\boldsymbol{v}_1^\ell$
			\ENDFOR

			\FOR {$t=2,\ldots,T$}
			\IF {$t>2$}
			\FOR {$\ell=1:L$}
			\STATE$\boldsymbol{z}_t^\ell=\boldsymbol{A}_\ell^T\boldsymbol{u}_t^\ell-\sum_{i=1}^{t-1} \hat{\beta}_{ti}^\ell\boldsymbol{v}_i^\ell$
			\ENDFOR
			\ENDIF
			\STATE$\boldsymbol{v}_t^1=\tilde{f}_t^1(\boldsymbol{v}_1^1,\boldsymbol{z}_1^1,\ldots,\boldsymbol{z}_t^1,\boldsymbol{E}_1)$			\STATE$\boldsymbol{m}_t^1=\boldsymbol{A}_1\boldsymbol{v}_t^1-\sum_{i=1}^t\hat{\alpha}_{ti}^1\boldsymbol{u}_i^1$
			\FOR {$\ell=2:L$}
			\STATE$\boldsymbol{v}_t^\ell=\tilde{f}_t^\ell(\boldsymbol{v}_1^\ell,\boldsymbol{z}_1^\ell,\ldots,\boldsymbol{z}_t^\ell,\boldsymbol{m}_1^{\ell-1},\ldots,\boldsymbol{m}_t^{\ell-1},\boldsymbol{E}_\ell)$	
			\STATE$\boldsymbol{m}_t^\ell=\boldsymbol{A}_\ell\boldsymbol{v}_t^\ell-\sum_{i=1}^t\hat{\alpha}_{ti}^\ell\boldsymbol{u}_i^\ell$
			\ENDFOR
			\FOR {$\ell=1:L-1$}
			\STATE$\boldsymbol{u}_{t+1}^\ell=\tilde{h}_{t+1}^\ell(\boldsymbol{m}_1^\ell,\ldots,\boldsymbol{m}_t^\ell,\boldsymbol{z}_1^{\ell+1},\ldots,\boldsymbol{z}_t^{\ell+1},\boldsymbol{F}_\ell)$
			\ENDFOR
			\STATE$\boldsymbol{u}_{t+1}^L=\tilde{h}_{t+1}^L(\boldsymbol{m}_1^L,\ldots,\boldsymbol{m}_t^L,\boldsymbol{F}_L)$
			\ENDFOR
		\end{algorithmic}
	\end{algorithm}  
	The auxiliary matrices $\boldsymbol{\hat{\Psi}}_t^\ell\in\mathbb{R}^{t\times t}$ and $\boldsymbol{\hat{\Phi}}_t^\ell\in\mathbb{R}^{(t+1)\times(t+1)}$ are defined as
	
	\begin{equation}
		\boldsymbol{\hat{\Psi}}_t^\ell=\left(\begin{matrix}
			\langle\partial_1{\boldsymbol{v}}_1^\ell\rangle & 0  & \cdots   & 0 & 0 \\
			\langle\partial_1{\boldsymbol{v}}_2^\ell\rangle & \langle\partial_2{\boldsymbol{v}}_2^\ell\rangle  & 0   & \cdots & 0\\
			\langle\partial_1{\boldsymbol{v}}_3^\ell\rangle & \langle\partial_2{\boldsymbol{v}}_3^\ell\rangle  & \langle\partial_3{\boldsymbol{v}}_3^\ell\rangle   & \cdots & 0\\
			\vdots & \vdots  & \vdots  & \ddots   & \vdots & \\
			\langle\partial_1{\boldsymbol{v}}_t^\ell\rangle & \langle\partial_2{\boldsymbol{v}}_t^\ell\rangle  &\langle\partial_3{\boldsymbol{v}}_t^\ell\rangle& \cdots\  & \langle\partial_t{\boldsymbol{v}}_t^\ell\rangle & \\
		\end{matrix}\right),
	\end{equation}
	\begin{equation}
		\boldsymbol{\hat{\Phi}}_t^\ell=\left(\begin{matrix}
			0 & 0  & \cdots   & 0 & 0 \\
			\langle\partial_1{\boldsymbol{u}}_2^\ell\rangle & 0  & 0   & \cdots & 0\\
			\langle\partial_1{\boldsymbol{u}}_3^\ell\rangle & \langle\partial_2{\boldsymbol{u}}_3^\ell\rangle  & 0   & \cdots & 0\\
			\vdots & \vdots  & \vdots  & \ddots   & \vdots & \\
			\langle\partial_1{\boldsymbol{u}}_t^\ell\rangle & \langle\partial_2{\boldsymbol{u}}_t^\ell\rangle & \cdots\  & \langle\partial_t{\boldsymbol{u}}_t^\ell\rangle & 0\\
		\end{matrix}\right),
	\end{equation}
	where $\partial_k{\boldsymbol{v}}_t^\ell$ and $\partial_k{\boldsymbol{u}}_{t+1}^\ell$ denote the row-wise partial derivatives $\partial_{z_k^\ell}\tilde{f}_t^\ell(\boldsymbol{v}_1^\ell(\boldsymbol{z}_1^\ell),\boldsymbol{z}_1^\ell,\ldots,\boldsymbol{z}_t^\ell,\boldsymbol{m}_1^{\ell-1},\ldots,\boldsymbol{m}_t^{\ell-1},\boldsymbol{E}_\ell)$ and $\partial_{m_k^\ell}\tilde{h}_{t+1}^\ell(\boldsymbol{m}_1^\ell,\ldots,\boldsymbol{m}_t^\ell,\boldsymbol{z}_1^{\ell+1},\ldots,\boldsymbol{z}_t^{\ell+1},\boldsymbol{F}_\ell)$ (resp.\ $\partial_{z_k^\ell}\tilde{f}_t^1(\boldsymbol{v}_1^1(\boldsymbol{z}_1^\ell),\boldsymbol{z}_1^1,\ldots,\boldsymbol{z}_t^1,\boldsymbol{E}_1)$ and   $\partial_{m_k^\ell}\tilde{h}_{t+1}^L(\boldsymbol{m}_1^L,\ldots,\boldsymbol{m}_t^L,\boldsymbol{F}_L)$ for the first and last layers). The argument $\boldsymbol{v}_1^\ell(\boldsymbol{z}_1^\ell)$ disappears for $t=1$.
	
	Next, we define matrices $\boldsymbol{\hat{M}}_t^{\alpha,\ell},\boldsymbol{\hat{M}}_t^{\beta,\ell}\in\mathbb{R}^{(t+1)\times(t+1)}$ as
	\begin{equation}
		\boldsymbol{\hat{M}}_t^{\alpha,\ell}=\sum_{j=0}^t\kappa_{2(j+1)}^\ell\boldsymbol{\hat{\Psi}}^\ell_t(\boldsymbol{\hat{\Phi}}^\ell_t\boldsymbol{\hat{\Psi}}^\ell_t)^j,
	\end{equation}
	\begin{equation}
		\boldsymbol{\hat{M}}_t^{\beta,\ell}=\delta_\ell\sum_{j=0}^t\kappa_{2(j+1)}^\ell\boldsymbol{\hat{\Phi}}^\ell_t(\boldsymbol{\hat{\Psi}}^\ell_t\boldsymbol{\hat{\Phi}}^\ell_t)^j.
	\end{equation}
	The coefficients $\{\hat{\alpha}_{ti}^\ell\}_{i=1}^t$ and $\{\hat{\beta}_{ti}^\ell\}_{i=1}^{t-1}$ are then obtained from the last rows of $\boldsymbol{\hat{M}}_t^{\alpha,\ell}$ and $\boldsymbol{\hat{M}}_{t+1}^{\beta,\ell}$ as
	\begin{equation}
		(\hat{\alpha}_{t1},\ldots,\hat{\alpha}_{tt})=[\boldsymbol{\hat{M}}_t^{\alpha,\ell}]_{t,1:t},
	\end{equation}
	\begin{equation}
		(\hat{\beta}_{t1},\ldots,\hat{\beta}_{t,t-1})=[\boldsymbol{\hat{M}}_t^{\beta,\ell}]_{t,1:t-1}.
	\end{equation}
	
	\begin{algorithm}[tb]   
		\caption{State evolution of ML-RI-General Recursion}
		\begin{algorithmic}[1]\label{algo:genSE}
			\STATE$Z_1^1\sim\mathcal{N}(0,\tilde{\Omega}_1^1)$
			\STATE$V_1^1=\tilde{f}_t^1(Z_1^1,E_1)$	
			\STATE$M_1^1\sim\mathcal{N}(0,\tilde{\Sigma}_1^1)$
			\STATE$V_1^1=\tilde{f}_1^1(Z_1^1,E_1)$
			\FOR {$\ell=2:L$}
			\STATE$Z_1^\ell\sim\mathcal{N}(0,\tilde{\Omega}_1^\ell)$
			\STATE$V_1^\ell=\tilde{f}_t^\ell(Z_1^\ell,M_1^{\ell-1},E_\ell)$
			\STATE$M_1^\ell\sim\mathcal{N}(0,\tilde{\Sigma}_1^\ell)$
			\ENDFOR
			\STATE$U_2^L=\tilde{h}_2^L(M_1^L,F_L)$
			\STATE$(Z_1^L,Z_2^L)\sim\mathcal{N}(0,\tilde{\Omega}_2^L)$
			\FOR {$\ell=L-1:1$}
			\STATE$U_2^\ell=\tilde{h}_2^\ell(M_1^\ell,Z_1^{\ell+1},Z_2^{\ell+1},F_L)$
			\STATE$(Z_1^\ell,Z_2^\ell)\sim\mathcal{N}(0,\tilde{\Omega}_2^\ell)$
			\ENDFOR
			
			\FOR {$t=2,\ldots,T$}
			\IF {$t>2$}
			\FOR {$\ell=1:L$}
			\STATE$(Z_1^\ell,\ldots,Z_t^\ell)\sim\mathcal{N}(0,\tilde{\Omega}_t^\ell)$
			\ENDFOR
			\ENDIF
			\STATE$V_t^1=\tilde{f}_t^1(V_1^1,Z_1^1,\ldots,Z_t^1,E_1)$	
			\STATE$(M_1^1,\ldots,M_t^1)\sim\mathcal{N}(0,\tilde{\Sigma}_t^1)$
			\FOR {$\ell=2:L$}
			\STATE$V_t^\ell=\tilde{f}_t^\ell(V_1^\ell,Z_1^\ell,\ldots,Z_t^\ell,M_1^{\ell-1},\ldots,M_t^{\ell-1},E_\ell)$
			\STATE$(M_1^\ell,\ldots,M_t^\ell)\sim\mathcal{N}(0,\tilde{\Sigma}_t^\ell)$
			\ENDFOR
			\FOR {$\ell=1:L-1$}
			\STATE$U_{t+1}^\ell=\tilde{h}_{t+1}^\ell(M_1^\ell,\ldots,M_t^\ell,Z_1^{\ell+1},\ldots,Z_t^{\ell+1},F_\ell)$
			\ENDFOR
			\STATE$U_{t+1}^L=\tilde{h}_{t+1}^L(M_1^L,\ldots,M_t^L,,F_L)$
			\ENDFOR
		\end{algorithmic}
	\end{algorithm} 
	
	The state evolution of the ML-RI-General recursion is described in Algorithm \ref{algo:genSE}. It requires auxiliary matrices $\boldsymbol{\tilde{\Psi}}_t^\ell,\boldsymbol{\tilde{\Phi}}_t^\ell,\boldsymbol{\tilde{\Gamma}}_t^\ell,\boldsymbol{\tilde{\Delta}}_t^\ell\in\mathbb{R}^{t\times t}$ defined as
	\begin{equation}
		\boldsymbol{\tilde{\Psi}}_t^\ell=\left(\begin{matrix}
			\mathbb{E}[\partial_1 V_1^\ell] & 0  & \cdots   & 0 & 0 \\
			\mathbb{E}[\partial_1{V}_2^\ell] & \mathbb{E}[\partial_2{V}_2^\ell]  & 0   & \cdots & 0\\
			\mathbb{E}[\partial_1{V}_3^\ell] & \mathbb{E}[\partial_2{V}_3^\ell] & \mathbb{E}[\partial_3{V}_3^\ell]  & \cdots & 0\\
			\vdots & \vdots  & \vdots  & \ddots   & \vdots & \\
			\mathbb{E}[\partial_1{V}_t^\ell] & \mathbb{E}[\partial_2{V}_t^\ell] &\mathbb{E}[\partial_3{V}_t^\ell]& \cdots\  & \mathbb{E}[\partial_t{V}_t^\ell] & \\
		\end{matrix}\right),
	\end{equation}
	\begin{equation}
		\boldsymbol{\tilde{\Phi}}_t^\ell=\left(\begin{matrix}
			0 & 0  & \cdots   & 0 & 0 \\
			\mathbb{E}[\partial_1{U}_2^\ell] & 0  & 0   & \cdots & 0\\
			\mathbb{E}[\partial_1{U}_3^\ell] & \mathbb{E}[\partial_2{U}_3^\ell]  & 0   & \cdots & 0\\
			\vdots & \vdots  & \vdots  & \ddots   & \vdots & \\
			\mathbb{E}[\partial_1{U}_t^\ell] & \mathbb{E}[\partial_2{U}_t^\ell] & \cdots\  & \mathbb{E}[\partial_t{U}_t^\ell] & 0\\
		\end{matrix}\right),
	\end{equation}
	\begin{equation}
		\boldsymbol{\tilde{\Gamma}}_t^\ell=\left(\begin{matrix}
			\mathbb{E}[(V_1^\ell)^2] & \mathbb{E}[V_1^\ell V_2^\ell]   & \cdots & \mathbb{E}[V_1^\ell V_t^\ell] \\
			\mathbb{E}[V_1^\ell V_2^\ell] & \mathbb{E}[(V_2^\ell)^2]    & \cdots & \mathbb{E}[V_2^\ell V_t^\ell]\\
			\vdots & \vdots & \ddots   & \vdots & \\
			\mathbb{E}[V_1^\ell V_t^\ell] & \mathbb{E}[V_2^\ell V_t^\ell] & \cdots\  & \mathbb{E}[(V_t^\ell)^2]
		\end{matrix}\right),
	\end{equation}
	\begin{equation}
		\boldsymbol{\tilde{\Delta}}_t^\ell=\left(\begin{matrix}
			\mathbb{E}[(U_1^\ell)^2]  & \mathbb{E}[U_1^\ell U_2^\ell]   & \cdots & \mathbb{E}[U_1^\ell U_t^\ell]\\
			\mathbb{E}[U_1^\ell U_2^\ell] & \mathbb{E}[(U_2^\ell)^2]   & \cdots & \mathbb{E}[U_2^\ell U_t^\ell]\\
			\vdots  & \vdots  & \ddots   & \vdots & \\
			\mathbb{E}[U_1^\ell U_t^\ell] &\mathbb{E}[U_2^\ell U_t^\ell]& \cdots\  & \mathbb{E}[(U_t^\ell)^2]
		\end{matrix}\right).
	\end{equation}
	We can then obtain the covariance matrices $\boldsymbol{\tilde{\Sigma}}_{t+1}^\ell,\boldsymbol{\tilde{\Omega}}_{t+1}^\ell\in\mathbb{R}^{t\times t}$ from
	\begin{equation}
		\boldsymbol{\tilde{\Sigma}}_{t+1}^\ell=\sum_{j=0}^{2t+1}\bar{\kappa}_{2(j+1)}^\ell\boldsymbol{\tilde{\Xi}}_{t+1}^{(j),\ell},
	\end{equation}
	\begin{equation}
		\boldsymbol{\tilde{\Omega}}_{t+1}^\ell=\delta_\ell\sum_{j=0}^{2(t+1)}\bar{\kappa}_{2(j+1)}^\ell\boldsymbol{\tilde{\Theta}}_{t+1}^{(j),\ell},
	\end{equation}
	where, for $j\geq1$,
	\begin{equation}
		\begin{aligned}
			&\boldsymbol{\tilde{\Xi}}_{t+1}^{(j),\ell}=\sum_{i=1}^j(\boldsymbol{\tilde{\Psi}}^\ell_{t+1}\boldsymbol{\tilde{\Phi}}^\ell_{t+1})^i\boldsymbol{\tilde{\Gamma}}^\ell_{t+1}((\boldsymbol{\tilde{\Psi}}^\ell_{t+1}\boldsymbol{\tilde{\Phi}}^\ell_{t+1})^T)^{j-i}\\
			&+\sum_{i=1}^{j-1}(\boldsymbol{\tilde{\Psi}}^\ell_{t+1}\boldsymbol{\tilde{\Phi}}^\ell_{t+1})^i\boldsymbol{\tilde{\Psi}}^\ell_{t+1}\boldsymbol{\tilde{\Delta}}^\ell_{t+1}(\boldsymbol{\tilde{\Psi}}^\ell_{t+1})^T((\boldsymbol{\tilde{\Psi}}^\ell_{t+1}\boldsymbol{\tilde{\Phi}}^\ell_{t+1})^T)^{j-i-1},
		\end{aligned}\label{xi}
	\end{equation}
	\begin{equation}
		\begin{aligned}
			&\boldsymbol{\tilde{\Theta}}_{t+1}^{(j),\ell}=\sum_{i=1}^j(\boldsymbol{\tilde{\Phi}}^\ell_{t+1}\boldsymbol{\tilde{\Psi}}^\ell_{t+1})^i\boldsymbol{\tilde{\Delta}}^\ell_{t+1}((\boldsymbol{\tilde{\Phi}}^\ell_{t+1}\boldsymbol{\tilde{\Psi}}^\ell_{t+1})^T)^{j-i}\\
			&+\sum_{i=1}^{j-1}(\boldsymbol{\tilde{\Phi}}^\ell_{t+1}\boldsymbol{\tilde{\Psi}}^\ell_{t+1})^i\boldsymbol{\tilde{\Phi}}^\ell_{t+1}\boldsymbol{\tilde{\Gamma}}^\ell_{t+1}(\boldsymbol{\tilde{\Psi}}_{t+1}^\ell)^T((\boldsymbol{\tilde{\Phi}}^\ell_{t+1}\boldsymbol{\tilde{\Psi}}^\ell_{t+1})^T)^{j-i-1},
		\end{aligned}\label{theta}
	\end{equation}
	with  $\boldsymbol{\tilde{\Xi}}_{t+1}^{(0),\ell}=\boldsymbol{\tilde{\Gamma}}_{t+1}^\ell$ and  $\boldsymbol{\tilde{\Theta}}_{t+1}^{(0),\ell}=\boldsymbol{\tilde{\Delta}}^\ell_{t+1}$.

	\newtheorem{assumption}{Assumption}
	\begin{assumption}
		(a) $\{\boldsymbol{A}_\ell\}_{\ell=1}^L$ are mutually independent rotationally invariant matrices, as defined in Section \ref{subsec:model}. (b) For the side information, $\boldsymbol{E}_\ell\overset{W_2}{\to}E_\ell$ and $\boldsymbol{F}_\ell\overset{W_2}{\to}F_\ell$ with finite second moments and independent from all other random variables. (c) For the initialization,  $\boldsymbol{u}_1^\ell\overset{W_2}{\to}U_1^\ell$ with finite second moment and independent from $\boldsymbol{A}_\ell$. (d) For the denoising functions, we assume that $\tilde{f}_t^\ell$ and $\tilde{h}_t^\ell$ are Lipschitz in each of their arguments, and that the partial derivatives $\partial_{z_k^\ell}\tilde{f}_t^\ell$, $\partial_{m_k^\ell}\tilde{h}_t^\ell$ are all continuous on sets of probability 1, under the laws given by Algorithm 4.
	\end{assumption}
	
	At this point, we are ready to state our key result, which is proved in Appendix \ref{app:thmgen}.
	
	\begin{theorem}\label{thm:gen}
		Consider the general recursion in Algorithm 3 and its state evolution in Algorithm 4 under Assumption 1. Let $\psi$: $\mathbb{R}^{2t+k_{1\ell}}\to\mathbb{R}$ and $\phi$: $\mathbb{R}^{2t+k_{2\ell}}\to\mathbb{R}$ be any pseudo-Lipschitz functions of order 2. Then for each $t\geq1$, we almost surely
		have
		\begin{equation}
			\begin{aligned}
				&\lim_{n_\ell\to\infty}\frac{1}{n_\ell}\sum_{i=1}^{n_\ell}\psi(z_{1,i}^\ell,\ldots,z_{t,i}^\ell,v_{1,i}^\ell,\ldots,v_{t,i}^\ell,E_{\ell,i})\\&\qquad\qquad\qquad=\mathbb E[\psi(Z_1^\ell,\ldots,Z_t^\ell,V_1^\ell,\ldots,V_t^\ell,E_\ell)],
			\end{aligned}
		\end{equation}
		\begin{equation}
			\begin{aligned}
				&\lim_{n_{\ell+1}\to\infty}\frac{1}{n_{\ell+1}}\sum_{i=1}^{n_{\ell+1}}\phi(m_{1,i}^\ell,\ldots,m_{t,i}^\ell,u_{1,i}^\ell,\ldots,u_{t+1,i}^\ell,F_{\ell,i})\\&\qquad\qquad\qquad =\mathbb E[\phi(M_1^\ell,\ldots,M_t^\ell,U_1^\ell,\ldots,U_{t+1}^\ell,F_\ell)],
			\end{aligned}
		\end{equation}
		where $E_{\ell,i}$ and $F_{\ell,i}$ denote the $i$-th row of $\boldsymbol{E}_\ell$ and $\boldsymbol{F}_\ell$.
	\end{theorem}
	
	We remark that the result above is equivalent to the following convergence in Wasserstein-2 distance: 
	\begin{equation*}
		\begin{split}
			&(\boldsymbol{z}_1^\ell,\ldots,\boldsymbol{z}_t^\ell,\boldsymbol{v}_1^\ell,\ldots,\boldsymbol{v}_t^\ell,\boldsymbol{E}_\ell)\overset{W_2}{\to}(Z_1^\ell,\ldots,Z_t^\ell,V_1^\ell,\ldots,V_t^\ell,E_\ell),\\
			&(\boldsymbol{m}_1^\ell,...,\boldsymbol{m}_t^\ell,\boldsymbol{u}_1^\ell,...,\boldsymbol{u}_t^\ell,\boldsymbol{F}_\ell)\overset{W_2}{\to}(M_1^\ell,...,M_t^\ell,U_1^\ell,...,U_t^\ell,F_\ell).
		\end{split}
	\end{equation*}

	\subsection{Proof of Theorem \ref{thm:gen}}\label{app:thmgen}
	%Following the argument in \cite{fan2022approximate}, 
	We start by defining $\boldsymbol{\hat{H}}_t^{(2k),\ell},\boldsymbol{\hat{I}}_t^{(2k),\ell},\boldsymbol{\hat{J}}_t^{(2k+1),\ell},\boldsymbol{\hat{L}}_t^{(2k),\ell}\in\mathbb{R}^{t\times t}$ as
	\begin{equation}
		\begin{aligned}
			\boldsymbol{\hat{H}}_t^{(2k),\ell}=\sum_{j=0}^\infty c_{2k,j}^\ell\boldsymbol{\hat{\Theta}}_t^{(j),\ell},\,\, \boldsymbol{\hat{I}}_t^{(2k+1),\ell}=\sum_{j=0}^\infty c_{2k+1,j}^\ell\boldsymbol{\hat{X}}_t^{(j),\ell},\\
			\boldsymbol{\hat{J}}_t^{(2k+1),\ell}=\sum_{j=0}^\infty c_{2k+1,j}^{\prime\ell}(\boldsymbol{\hat{X}}_t^{(j),\ell})^T,\,\, \boldsymbol{\hat{L}}_t^{(2k),\ell}=\sum_{j=0}^\infty c_{2k,j}^{\prime\ell}\boldsymbol{\hat{\Xi}}_t^{(j),\ell},
		\end{aligned}
	\end{equation}
	where $\boldsymbol{\hat{\Theta}}$ and $\boldsymbol{\hat{\Xi}}$ are the empirical versions of $\boldsymbol{\tilde{\Theta}}$ and $\boldsymbol{\tilde{\Xi}}$. We use $\ \hat{}\ $ to denote an empirical average and $\ \tilde{}\ $ or $\ \bar{}\ $ to denote an expectation. For example, given $\boldsymbol{\tilde{\Xi}}$, $\boldsymbol{\hat{\Xi}}$ is obtained by replacing $\mathbb{E}[\cdot]$ with $\langle\cdot\rangle$ and $\bar{\kappa}$ with $\kappa$. Furthermore, we define %Other variables are defined as
	\begin{equation}
		\begin{aligned}
			&c_{2k+1,j}^\ell=\sum_{m=0}^{j+1}c_{2k,m}^\ell\kappa_{2(j+1-m)}^\ell,\\
			& 	c_{2k+1,j}^{\prime\ell}=\delta_\ell\sum_{m=0}^{j+1}c_{2k,m}^{\prime\ell}\kappa_{2(j+1-m)}^\ell,\\
			&c_{2k+2,j}^\ell=\delta_\ell\sum_{m=0}^jc_{2k+1,m}^\ell\kappa_{2(j-m)}^\ell,\\
			& c_{2k+2,j}^{\prime\ell}=\sum_{m=0}^jc_{2k+1,m}^{\prime\ell}\kappa_{2(j-m)}^\ell,
		\end{aligned}
	\end{equation}
	and
	\begin{equation}
		\begin{aligned}
			\boldsymbol{\hat{X}}_t^{(j),\ell}=&\sum_{i=1}^j(\boldsymbol{\hat{\Psi}}^\ell_t\boldsymbol{\hat{\Phi}}^\ell_t)^i\boldsymbol{\hat{\Psi}}^\ell_t\boldsymbol{\hat{\Delta}}^\ell_t((\boldsymbol{\hat{\Psi}}^\ell_t\boldsymbol{\hat{\Phi}}^\ell_t)^T)^{j-i}\\
			&+\sum_{i=1}^{j-1}(\boldsymbol{\hat{\Psi}}^\ell_t\boldsymbol{\hat{\Phi}}^\ell_t)^i\boldsymbol{\hat{\Gamma}}^\ell_t(\boldsymbol{\hat{\Phi}}^\ell_t)^T((\boldsymbol{\hat{\Psi}}^\ell_t\boldsymbol{\hat{\Phi}}^\ell_t)^T)^{j-i},
		\end{aligned}
	\end{equation}
	with $\kappa_0^\ell=0$, $c_{0,0}^\ell=c_{0,0}^{\prime\ell}=1$, $c_{0,j}^\ell=c_{0,j}^{\prime\ell}=1$. As shown in Lemma B.1 of \cite{fan2022approximate}, $c_{1,j}^\ell=\kappa_{2(j+1)}^\ell$ and $c_{2k,0}^\ell=m_{2k}^\ell$. %, the $2k$-th moment of $\boldsymbol{\Lambda}_\ell$. 
	We further define $\boldsymbol{\hat{\gamma}}_t^\ell,\boldsymbol{\hat{T}}_t^\ell\in\mathbb{R}^{2t\times2t}$ as
	\begin{equation}
		\boldsymbol{\hat{\gamma}}_t^\ell=\left(\begin{matrix}
			\boldsymbol{\hat{\Delta}}_t^\ell&\boldsymbol{\hat{\Delta}}_t^\ell\boldsymbol{\hat{M}}_t^{\alpha,\ell}+\boldsymbol{\hat{\Phi}}_t^\ell\boldsymbol{\hat{\Sigma}}_t^\ell\\
			(\boldsymbol{\hat{\Phi}}_t^\ell)^T&(\boldsymbol{\hat{\Phi}}_t^\ell)^T\boldsymbol{\hat{M}}_t^{\alpha,\ell}+\boldsymbol{I}
		\end{matrix}\right),
	\end{equation}
	\begin{equation}
		\boldsymbol{\hat{T}}_t^\ell=\left(\begin{matrix}
			\boldsymbol{\hat{\Gamma}}_t^\ell&\boldsymbol{\hat{\Gamma}}_t^\ell\boldsymbol{\hat{M}}_t^{\beta,\ell}+\boldsymbol{\hat{\Psi}}_t^\ell\boldsymbol{\hat{\Omega}}_t^\ell\\
			(\boldsymbol{\hat{\Psi}}_t^\ell)^T&(\boldsymbol{\hat{\Psi}}_t^\ell)^T\boldsymbol{\hat{M}}_t^{\beta,\ell}+\boldsymbol{I}
		\end{matrix}\right).
	\end{equation}
	It is also convenient to divide the main part in Algorithm 3 as
	\begin{align}
		&\boldsymbol{d}_t^\ell=\boldsymbol{O}_\ell\boldsymbol{u}_t^\ell,\quad \boldsymbol{e}_t^\ell=\boldsymbol{Q}_\ell^T\boldsymbol{\Lambda}_\ell^T\boldsymbol{d}_t^\ell,\\
		&\boldsymbol{p}_t^\ell=\boldsymbol{Q}_\ell\boldsymbol{v}_t^\ell,\quad \boldsymbol{q}_t^\ell=\boldsymbol{O}_\ell^T\boldsymbol{\Lambda}_\ell\boldsymbol{p}_t^\ell,
	\end{align}
	and thus
	\begin{equation}
		\boldsymbol{z}_t^\ell=\boldsymbol{e}_t^\ell-\sum_{i=1}^{t-1} \hat{\beta}_{ti}^\ell\boldsymbol{v}_i^\ell,\quad \boldsymbol{m}_t^\ell=\boldsymbol{q}_t^\ell-\sum_{i=1}^t\hat{\alpha}_{ti}^\ell\boldsymbol{u}_i^\ell.
	\end{equation}
	
	Before stating our main lemma, let us define $\boldsymbol{\lambda}_1^\ell\in\mathbb{R}^{n_{\ell+1}}$ as $\boldsymbol{\lambda}_1^\ell=\boldsymbol{\lambda}_\ell$ for $\delta_\ell\geq1$ and extended by additional $0$'s for $\delta_\ell<1$. Similarly, we define $\boldsymbol{\lambda}_2^\ell\in\mathbb{R}^{n_\ell}$ as $\boldsymbol{\lambda}_2^\ell=\boldsymbol{\lambda}_\ell$ for $\delta_\ell<1$ and extended by additional $0$'s for $\delta_\ell\geq1$. We also let $\Lambda_1^\ell, \Lambda_2^\ell$ be such that  $\boldsymbol{\lambda}_1^\ell\overset{W_2}{\to}\Lambda_1^\ell$ and $\boldsymbol{\lambda}_2^\ell\overset{W_2}{\to}\Lambda_2^\ell$. Now we are ready to state and prove Lemma \ref{lemma:main}.
	\newtheorem{lemma}{Lemma}
	\begin{lemma}\label{lemma:main}
		Under Assumption 1, the following results almost surely hold for $t \in [T]$ and $\ell\in [L]$:
		\begin{enumerate}[(a)]
			\item Let $\boldsymbol{\tilde{I}}_t^{\prime(2k),\ell}$ (resp.\ $\boldsymbol{\hat{I}}_t^{\prime(2k),\ell}$) be the matrix containing the  first $t-1$ rows of $\boldsymbol{\tilde{I}}_t^{(2k),\ell}$ (resp.\ $\boldsymbol{\hat{I}}_t^{(2k),\ell}$). Then,
			
			\begin{equation}
				\begin{split}
					\lim_{n_\ell\to\infty}(&\boldsymbol{\hat{\Delta}}_t^\ell,\boldsymbol{\hat{\Phi}}_t^\ell,\boldsymbol{\hat{\Theta}}_t^{(j),\ell},\boldsymbol{\hat{M}}_t^{\beta,\ell},\boldsymbol{\hat{\Omega}}_t^\ell,\boldsymbol{\hat{H}}_t^{(2k),\ell},\boldsymbol{\hat{I}}_t^{\prime(2k),\ell})\\
					&=(\boldsymbol{\tilde{\Delta}}_t^\ell,\boldsymbol{\tilde{\Phi}}_t^\ell,\boldsymbol{\tilde{\Theta}}_t^{(j),\ell},\boldsymbol{\tilde{M}}_t^{\beta,\ell},\boldsymbol{\tilde{\Omega}}_t^\ell,\boldsymbol{\tilde{H}}_t^{(2k),\ell},\boldsymbol{\tilde{I}}_t^{\prime(2k),\ell})  .  
				\end{split}
			\end{equation}
			\item For some random variables $D_1^\ell,\ldots,D_t^\ell,\bar{P}_1^\ell,\ldots,\bar{P}_{t-1}^\ell$ that satisfy $\mathbb{E}[(D_1^\ell,\ldots,D_t^\ell)^T(\Lambda_1^\ell)^{2k}(D_1^\ell,\ldots,D_t^\ell)]=\tilde{\boldsymbol{H}}_t^{(2k),\ell}$ and $\mathbb{E}[(\bar{P}_1^\ell,\ldots,\bar{P}_{t-1}^\ell)^T(\Lambda_1^\ell)^{2k}(D_1^\ell,\ldots,D_t^\ell)]=\tilde{\boldsymbol{I}}_t^{\prime(2k),\ell}$, we have
			\begin{equation}
				\begin{split}
					(\boldsymbol{d}_1^\ell,\ldots,\boldsymbol{d}_t^\ell,&\boldsymbol{\Lambda}_\ell\boldsymbol{p}_1^\ell,\ldots,\boldsymbol{\Lambda}_\ell\boldsymbol{p}_{t-1}^\ell,\boldsymbol{\lambda}_1^\ell)\\
					&\overset{W_2}{\to}(D_1^\ell,\ldots,D_t^\ell,\bar{P}_1^{\prime\ell},\ldots,\bar{P}_{t-1}^\ell,\Lambda_1^\ell).
				\end{split}
			\end{equation}
			
			\item As described in Theorem \ref{thm:gen},
			\begin{equation}
				\begin{split}
					(\boldsymbol{v}_1^\ell,\ldots,\boldsymbol{v}_t^\ell,&\boldsymbol{z}_1^\ell,\ldots,\boldsymbol{z}_t^\ell,\boldsymbol{E}_\ell)\\
					&\overset{W_2}{\to}(V_1^\ell,\ldots,V_t^\ell,Z_1^\ell,\ldots,Z_t^\ell,E_\ell).
				\end{split}
			\end{equation}
			
			\item We have that
			\begin{equation}
				\begin{split}
					\lim_{n_\ell\to\infty}(&\hat{\boldsymbol{\Gamma}}_t^\ell,\hat{\boldsymbol{\Psi}}_t^\ell,\hat{\boldsymbol{\Xi}}_t^{(j),\ell},\hat{\boldsymbol{M}}_t^{\alpha,\ell},\hat{\boldsymbol{\Sigma}}_t^\ell,\hat{\boldsymbol{L}}_t^{(2k),\ell},\hat{\boldsymbol{J}}_t^{(2k+1),\ell})\\&=(\tilde{\boldsymbol{\Gamma}}_t^\ell,\tilde{\boldsymbol{\Psi}}_t^\ell,\tilde{\boldsymbol{\Xi}}_t^{(j),\ell},\tilde{\boldsymbol{M}}_t^{\alpha,\ell},\tilde{\boldsymbol{\Sigma}}_t^\ell,\tilde{\boldsymbol{L}}_t^{(2k),\ell},\tilde{\boldsymbol{J}}_t^{(2k+1),\ell}).
				\end{split}
			\end{equation}
			
			\item For some random variables $P_1^\ell,\ldots,P_t^\ell,\bar{D}_1^\ell,\ldots,\bar{D}_t^\ell$ that satisfy $\mathbb{E}[(P_1^\ell,\ldots,P_t^\ell)^T(\Lambda_2^\ell)^{2k}(P_1^\ell,\ldots,P_t^\ell)]=\tilde{\boldsymbol{L}}_t^{(2k),\ell}$ and $\mathbb{E}[(\bar{D}_1^\ell,\ldots,\bar{D}_t^\ell)^T(\Lambda_2^\ell)^{2k}(P_1^\ell,\ldots,P_t^\ell)]=\tilde{\boldsymbol{J}}_t^{(2k+1),\ell}$, we have
			\begin{equation}
				\begin{split}
					(\boldsymbol{p}_1^\ell,\ldots,\boldsymbol{p}_t^\ell,&\boldsymbol{\Lambda}_\ell^T\boldsymbol{d}_1^\ell,\ldots,\boldsymbol{\Lambda}_\ell^T\boldsymbol{d}_t^\ell,\boldsymbol{\lambda}_2^\ell)\\&\overset{W_2}{\to}(P_1^\ell,\ldots,P_t^\ell,\bar{D}_1^\ell,\ldots,\bar{D}_t^\ell,\Lambda_2^\ell).
				\end{split}
			\end{equation}
			
			\item As described in Theorem \ref{thm:gen},
			\begin{equation}
				\begin{split}
					(\boldsymbol{u}_1^\ell,&\ldots,\boldsymbol{u}_t^\ell,\boldsymbol{m}_1^\ell,\ldots,\boldsymbol{m}_t^\ell,\boldsymbol{F}_\ell)\\&\overset{W_2}{\to}(U_1^\ell,\ldots,U_t^\ell,M_1^\ell,\ldots,M_t^\ell,F_\ell).
				\end{split}
			\end{equation}
			
			\item As described in Theorem \ref{thm:gen},
			\begin{equation}
				\begin{split}
					(\boldsymbol{u}_1^\ell,&\ldots,\boldsymbol{u}_{t+1}^\ell,\boldsymbol{m}_1^\ell,\ldots,\boldsymbol{m}_t^\ell,\boldsymbol{F}_\ell)\\&\overset{W_2}{\to}(U_1^\ell,\ldots,U_{t+1}^\ell,M_1^\ell,\ldots,M_t^\ell,F_\ell).
				\end{split}
			\end{equation}	
			
			\item $\left(\begin{matrix}
				\boldsymbol{\tilde{\Delta}}_t^\ell&\boldsymbol{\tilde{\Phi}}_t^\ell\boldsymbol{\tilde{\Sigma}}_t^\ell\\
				\boldsymbol{\tilde{\Sigma}}_t^\ell(\boldsymbol{\tilde{\Phi}}_t^\ell)^T&\boldsymbol{\tilde{\Sigma}}^\ell_t
			\end{matrix}\right)$ and $\left(\begin{matrix}
				\boldsymbol{\tilde{\Gamma}}_t^\ell&\boldsymbol{\tilde{\Psi}}_t^\ell\boldsymbol{\tilde{\Omega}}_t^\ell\\
				\boldsymbol{\tilde{\Omega}}_t^\ell(\boldsymbol{\tilde{\Psi}}_t^\ell)^T&\boldsymbol{\tilde{\Omega}}_t^\ell
			\end{matrix}\right)$ are non-singular.
		\end{enumerate}	
	\end{lemma}
	\begin{proof}
		To induct on $t$ and $\ell$, we use symbols like $t^{(a),\ell}$, $t^{(b),\ell}$ to denote Lemma 1(a), Lemma 1(b) for iteration $t$ and layer $\ell$, etc. We will skip the parts that are similar to Lemma B.4 of \cite{fan2022approximate}. Our induction sequence follows exactly the iteration of Algorithm \ref{algo:gen}, that is $1^{(a),1}\to1^{(b),1}\to\ldots\to1^{(f),1}\to\ldots\to1^{(a),L}\ldots\to1^{(e),L}\to1^{(f),L}\to1^{(g),L}\to1^{(h),L}\to2^{(a),L}\to2^{(b),L}\to2^{(c),L}\to\ldots\to1^{(g),1}\to\ldots\to2^{(b),1}\to2^{(c),1}\to\ldots\to(t+1)^{(a),1}\to(t+1)^{(b),1}\to\ldots\to(t+1)^{(f),1}\to\ldots\to(t+1)^{(a),L}\to\ldots\to(t+1)^{(f),L}\to(t+1)^{(g),1}\to(t+1)^{(h),1}\to\ldots\to(t+1)^{(g),L}\to(t+1)^{(h),L}$. All the following limits hold almost surely, so we will not specify.
		
		$1^{(a),1}$: $\boldsymbol{\hat{\Delta}}_1^1=\langle(\boldsymbol{u}_1^1)^2\rangle$, and $\boldsymbol{\tilde{\Delta}}_1^\ell=\mathbb{E}[(U_1^1)^2]$. According to Assumption 2(c), $\boldsymbol{\hat{\Delta}}_1^1\to\boldsymbol{\tilde{\Delta}}_1^1$. The other variables are $0$ for $t=1$, so $1^{(a),1}$ naturally holds.
		
		$1^{(b),1}$: $\boldsymbol{d}_1^1={\boldsymbol O}_1\boldsymbol{u}_1^1$. According to Assumption 2(a) and Lemma \ref{useful lemma3}, $(\boldsymbol{d}_1^1,\boldsymbol{\lambda}_1^1)\to(D_1^1,\Lambda_1^1)$, with $D_1^1\sim\mathcal{N}(0,\mathbb{E}[(U_1^1)^2])$ independent of $\Lambda_1^1$ and $\mathbb{E}[(\Lambda_1^1)^{2k}(D_1^1)^2]=\bar{m}_{2k}^1\mathbb{E}[(U_1^1)^2]$, so $1^{(b),1}$ holds.
		
		$1^{(c),1}$: $\boldsymbol{z}_1^1=\boldsymbol{Q}_1^T\boldsymbol{\Lambda}_\ell^T\boldsymbol{d}_1^1$, and from $1^{(b),1}$ we have $\frac{1}{n_\ell}||\boldsymbol{\Lambda}_\ell^T\boldsymbol{d}_1^1||^2\to\delta_\ell\bar{\kappa}_2^1\tilde{\Delta}_1^1$. According to Lemma \ref{useful lemma1} and Lemma \ref{useful lemma3}  with $\boldsymbol{v}_1^1=\tilde{f}_1^1(\boldsymbol{z}_1^1,E_1)$ and Assumption 2(d), $(\boldsymbol{v}_1^1, \boldsymbol{z}_1^1,\boldsymbol{E}_1)\overset{W_2}{\to}(V_1^1,Z_1^1,E_1)$ where $Z_1^1\sim\mathcal{N}(0,\delta_1\bar{\kappa}_2^1\tilde{\Delta}_1^1)$ independent of $E_1$.
		
		$1^{(d),1}$: $\boldsymbol{\hat{\Gamma}}_1^1=\langle(\boldsymbol{v}_1^1)^2\rangle$, and $\boldsymbol{\tilde{\Gamma}}_1^1=\mathbb{E}[(V_1^1)^2]$, $\boldsymbol{\hat{\Psi}}_1^1=\langle\partial_1\boldsymbol{v}_1^1(\boldsymbol{z}_1,\boldsymbol{E}_1)\rangle$, and $\boldsymbol{\tilde{\Psi}}_1^1=\mathbb{E}[\partial_1\boldsymbol{v}_1(\boldsymbol{z}_1,\boldsymbol{E}_1)]$. According to Lemma \ref{useful lemma1}, $\boldsymbol{\hat{\Gamma}}_1^1\to\boldsymbol{\tilde{\Gamma}}_1^1$ and $\boldsymbol{\hat{\Psi}}_1^1\to\boldsymbol{\tilde{\Psi}}_1^1$. The other variables are $0$ for $t=1$, so $1^{(d),1}$ naturally holds.
		
		$1^{(e),1},1^{(f),1}$: Note that the law of $\boldsymbol{Q}_1$ is conditioned on $\boldsymbol{\Lambda}_1\boldsymbol{d}_1^1=\boldsymbol{Q}_1\boldsymbol{z}_1^1$, and the law of $\boldsymbol{O}_1$ is conditioned on $\boldsymbol{r}_1^1=\boldsymbol{O}_1\boldsymbol{u}_1^1$. Thus, the claim follows from Lemma B.4 of \cite{fan2022approximate}.
		
		Next we will assume that Lemma 1(a)-(f) hold for $t=1$ and layers up to (including) $\ell-1$.
		
		$1^{(a),\ell}$, $1^{(b),\ell}$: Same argument as for $1^{(a),1}$ and $1^{(b),1}$.
		
		$1^{(c),\ell}$: As in $1^{(c),1}$, we have $( \boldsymbol{z}_1^\ell,\boldsymbol{E}_\ell)\overset{W_2}{\to}(Z_1^\ell,E_\ell)$ where $Z_1^\ell\sim\mathcal{N}(0,\delta_1\bar{\kappa}_2^1\boldsymbol{\tilde{\Delta}}_1^\ell)$ is independent of $E_\ell$. Algorithm 3 implies that $\boldsymbol{v}_1^\ell=\tilde{f}_1^\ell(\boldsymbol{z}_1^\ell,\boldsymbol{m}_1^{\ell-1},\boldsymbol{E}_\ell)$, and $1^{(f),\ell-1}$ implies that $\boldsymbol{m}_1^{\ell-1}\overset{W_2}{\to}M_1^{\ell-1}$ independent of $Z_1^\ell$. According to Lemma \ref{useful lemma1}, $(\boldsymbol{v}_1^\ell, \boldsymbol{z}_1^\ell,\boldsymbol{E}_\ell)\overset{W_2}{\to}(V_1^\ell, Z_1^\ell,E_\ell)$.
		
		$1^{(d),\ell},1^{(e),\ell}$: Same argument as $1^{(d),1}$, $1^{(e),1}$.
		
		$1^{(f),\ell}$: Conditioned on $\boldsymbol{r}_1^\ell=\boldsymbol{O}_\ell\boldsymbol{u}_1^\ell$, the distribution of $\boldsymbol{q}_1^\ell$ can be expressed by $\boldsymbol{q}_1^\ell=\boldsymbol{O}_\ell^T\boldsymbol{\Lambda}_\ell\boldsymbol{p}_1^\ell=\boldsymbol{q}_\parallel^\ell+\boldsymbol{q}_\perp^\ell$, where
		\begin{equation}
			\boldsymbol{q}_\parallel^\ell=\boldsymbol{u}_1^\ell((\boldsymbol{u}_1^\ell)^T\boldsymbol{u}_1^\ell)^{-1}(\boldsymbol{r}_1^\ell)^T\boldsymbol{\Lambda}\boldsymbol{p}_1^\ell\overset{W_2}{\to}\tilde{\alpha}_{11}^\ell U_1^\ell
		\end{equation}
		and $\boldsymbol{q}_\perp^\ell\overset{W_2}{\to}\mathcal{N}(0,\boldsymbol{\tilde{\Sigma}}_1^\ell)$ is independent of $Z_1^\ell$ and $V_1^\ell$. As $\boldsymbol{m}_1^\ell=\boldsymbol{q}_1^\ell-\hat{\alpha}_{11}^\ell\boldsymbol{u}_1^\ell$, we have
		\begin{equation}
			(\boldsymbol{u}_1^\ell,\boldsymbol{m}_1^\ell,\boldsymbol{F}_\ell)\overset{W_2}{\to}(U_1^\ell,M_1^\ell,F_\ell)
		\end{equation}
		independent of $
		\{Z_1^\ell,V_1^\ell\}_{\ell=1}^L$.
		
		Next, we will assume that Lemma 1(a)-(f) hold for $t=1$ and all layers.
		
		$1^{(g),L}$: As $\boldsymbol{u}_2^L=\tilde{h}_2^L(\boldsymbol{m}_1^L,\boldsymbol{F}_L)$, according to $1^{(f),L}$ and Lemma \ref{useful lemma1}, we have
		\begin{equation}
			(\boldsymbol{u}_1^L,\boldsymbol{u}_2^L,\boldsymbol{m}_1^L,\boldsymbol{F}_L)\overset{W_2}{\to}(U_1^L,U_2^L,M_1^L,F_L)
		\end{equation}
		independent of $
		\{Z_1^\ell,V_1^\ell\}_{\ell=1}^L$.
		
		$1^{(h),L}$: This follows from Lemma B.4 of \cite{fan2022approximate}.
		
		$2^{(a),L}$: This follows from Lemma B.4 of \cite{fan2022approximate}.
		
		$2^{(b),L}$: We denote $\boldsymbol{D}_t^\ell=[\boldsymbol{d}_1^\ell,\ldots,\boldsymbol{d}_t^\ell]$, $\boldsymbol{P}_t^\ell=[\boldsymbol{p}_1^\ell,\ldots,\boldsymbol{p}_t^\ell]$, $\boldsymbol{U}_t^\ell=[\boldsymbol{u}_1^\ell,\ldots,\boldsymbol{u}_t^\ell]$, $\boldsymbol{M}_t^\ell=[\boldsymbol{m}_1^\ell,\ldots,\boldsymbol{m}_t^\ell]$ and use $t=1$, $\ell=L$ herein. Then the distribution of $\boldsymbol{O}_\ell$ is conditioned on 
		\begin{equation}
			(\boldsymbol{D}_t^\ell,\boldsymbol{\Lambda}_\ell\boldsymbol{P}_t^\ell)\left(\begin{matrix}
				\boldsymbol{I}&-\hat{\boldsymbol{M}}_{t+1}^{\alpha,\ell}\\
				0&\boldsymbol{I}
			\end{matrix}\right)=\boldsymbol{O}_\ell(\boldsymbol{U}_t^\ell,\boldsymbol{M}_t^\ell).
		\end{equation}
		This gives that $\boldsymbol{d}_{t+1}^\ell=\boldsymbol{d}^\ell_\parallel+\boldsymbol{d}^\ell_\perp$, where the empirical distribution of $\boldsymbol{d}_\perp^\ell$ approaches a Gaussian random variable independent from previous steps (i.e., $\boldsymbol{D}_t^\ell$, $\boldsymbol{P}_t^\ell$, $\boldsymbol{U}_t^\ell$ and $\boldsymbol{M}_t^\ell$)
		\begin{equation}
			\begin{split}
				&\boldsymbol{d}_\perp^\ell\overset{W_2}{\to}\mathcal{N}\bigg(0,\\
				&\mathbb{E}[U_{t+1}]^2-
				\left(\begin{matrix}\tilde{\boldsymbol{\delta}}_t^\ell\\\tilde{\boldsymbol{\Sigma}}_t^\ell\tilde{\boldsymbol{\phi}}_t^\ell\end{matrix}\right)^T
				\left(\begin{matrix}
					\tilde{\boldsymbol{\Delta}}_t^\ell&\tilde{\boldsymbol{\Phi}}_t^\ell\tilde{\boldsymbol{\Sigma}}_t^\ell\\
					\tilde{\boldsymbol{\Sigma}}_t^\ell(\tilde{\boldsymbol{\Phi}}_t^\ell)^T&\tilde{\boldsymbol{\Sigma}}^\ell_t
				\end{matrix}\right)^{-1}
				\left(\begin{matrix}\tilde{\boldsymbol{\delta}}_t^\ell\\\tilde{\boldsymbol{\Sigma}}_t^\ell\tilde{\boldsymbol{\phi}}_t^\ell\end{matrix}\right)\bigg),\label{d_perp}
			\end{split}
		\end{equation}
		where $\tilde{\boldsymbol{\delta}}_t^\ell$ and $\tilde{\boldsymbol{\phi}}_t^\ell$ are the last columns of $\tilde{\boldsymbol{\Delta}}_{t+1}^\ell$ and $\tilde{\boldsymbol{\Phi}}_{t+1}^\ell$ with last entries removed, and 
		\begin{equation}
			\boldsymbol{d}^\ell_\parallel=(\boldsymbol{D}_t^\ell,\boldsymbol{\Lambda}_\ell\boldsymbol{P}_t^\ell)\left(\begin{matrix}
				\boldsymbol{I}&-\hat{\boldsymbol{M}}_{t+1}^{\alpha,\ell}\\
				0&\boldsymbol{I}
			\end{matrix}\right)\boldsymbol{B}_t^\ell n_{\ell+1}^{-1}\left(\begin{matrix}
				(\boldsymbol{U}_t^\ell)^T\\
				(\boldsymbol{U}_M^\ell)^T
			\end{matrix}\right)\boldsymbol{u}_{t+1}^\ell,
		\end{equation}
		where
		\begin{equation}
			\boldsymbol{B}_t^\ell=\frac{1}{n_{\ell+1}}\left(\begin{matrix}
				(\boldsymbol{U}_t^\ell)^T\boldsymbol{U}_t^\ell&(\boldsymbol{U}_t^\ell)^T\boldsymbol{M}_t^\ell\\
				(\boldsymbol{M}_t^\ell)^T\boldsymbol{U}_t^\ell&(\boldsymbol{M}_t^\ell)^T\boldsymbol{M}_t^\ell
			\end{matrix}\right).
		\end{equation}
		
		We then step back to $2^{(a),L}$. According to $1^{(f),L}$, $\boldsymbol{M}_1^L$ is composed of jointly Gaussian distributed columns, and $\boldsymbol{U}_1^L$ is a function of $\boldsymbol{M}_1^L$ and $\boldsymbol{E}_L$ (row-i.i.d.\ variables independent of $\boldsymbol{M}_t^L$). Hence, according to Stein's lemma,
		\begin{equation}
			\boldsymbol{B}_1^L\to\left(\begin{matrix}
				\tilde{\boldsymbol{\Delta}}_1^L&\tilde{\boldsymbol{\Phi}}_1^L\tilde{\boldsymbol{\Sigma}}_1^L\\
				\tilde{\boldsymbol{\Sigma}}_1^L(\tilde{\boldsymbol{\Phi}}_1^L)^T&\tilde{\boldsymbol{\Sigma}}_1^L
			\end{matrix}\right)
		\end{equation}
		and
		\begin{equation}
			\frac{1}{n_{L+1}}(\boldsymbol{U}_t^L)^T\boldsymbol{u}_{t+1}^L\to\tilde{\boldsymbol{\delta}}_t^L,\quad \frac{1}{n_{L+1}}(\boldsymbol{M}_t^L)^T\boldsymbol{u}_{t+1}^L\to\tilde{\boldsymbol{\phi}}_t^L.
		\end{equation}
		Thus, through arguments similar to \cite{fan2022approximate}, we can obtain 
		\begin{equation}
			\boldsymbol{d}^L_\parallel\overset{W_2}{\to}(D_1^L,\bar{P}_1^L)(\boldsymbol{\gamma}_1^L)^{-1}\left(\begin{matrix}
				\tilde{\boldsymbol{\delta}}_1^L\\
				\tilde{\boldsymbol{\phi}}_1^L
			\end{matrix}\right).
		\end{equation}
		As a result,  $\mathbb{E}[(D_1^L,D_2^L)^T(\Lambda_1^L)^{2k}(D_1^L,D_2^L)]=\tilde{\boldsymbol{H}}_2^{(2k),L}$ and $\mathbb{E}[\bar{P}_1^L(\Lambda_1^L)^{2k}(D_1^L,D_2^L)]=\tilde{\boldsymbol{I}}_2^{\prime(2k),L}$. %correspondingly.
		
		$2^{(c),L}$: We also generally consider step $t+1$ and layer $\ell$. The distribution of $\boldsymbol{Q}_\ell$ is conditioned on
		\begin{equation}
			(\boldsymbol{P}_t^\ell,\boldsymbol{\Lambda}_\ell^T\boldsymbol{D}_t^\ell)\left(\begin{matrix}
				\boldsymbol{I}&-\hat{\boldsymbol{M}}_{t+1}^{\beta,\ell}\\
				0&\boldsymbol{I}
			\end{matrix}\right)=\boldsymbol{Q}_\ell(\boldsymbol{V}_t^\ell,\boldsymbol{Z}_t^\ell),
		\end{equation}
		where $\boldsymbol{V}_t^\ell=[\boldsymbol{v}_1^\ell,\ldots,\boldsymbol{v}_t^\ell]$ and $\boldsymbol{Z}_t^\ell=[\boldsymbol{z}_1^\ell,\ldots,\boldsymbol{z}_t^\ell]$. Thus, $\boldsymbol{e}_{t+1}^\ell=\boldsymbol{e}_\parallel^\ell+\boldsymbol{e}_\perp^\ell$, where $\boldsymbol{e}_\perp^\ell$ approaches a Gaussian variable independent from previous steps:
		\begin{equation}
			\begin{aligned}
				\boldsymbol{e}_\perp^\ell\overset{W_2}{\to}&\mathcal{N}\left(0,\delta_\ell [\tilde{\boldsymbol{H}}]_{t+1,t+1}^{(2),\ell}-\right.\\
				&\left.\delta_\ell^2
				\left(\begin{matrix}
					\tilde{\boldsymbol{i}}_t^{(1),\ell}\\\tilde{\boldsymbol{h}}_t^{(2),\ell}
				\end{matrix}\right)^T
				\left(\begin{matrix}
					\tilde{\boldsymbol{L}}_t^{(0),\ell}&(\tilde{\boldsymbol{J}}_t^{(1)})^t\\
					\tilde{\boldsymbol{J}}_t^{(1),\ell}&\delta_\ell\boldsymbol{H}_t^{(2),\ell}
				\end{matrix}\right)^{-1}
				\left(\begin{matrix}
					\tilde{\boldsymbol{i}}_t^{(1),\ell}\\\tilde{\boldsymbol{h}}_t^{(2),\ell}
				\end{matrix}\right)
				\right),
			\end{aligned}\label{e_perp}
		\end{equation}
		where $\tilde{\boldsymbol{i}}_t^{(1),\ell}$ and $\tilde{\boldsymbol{h}}_t^{(2),\ell}$ denote the last rows of $\tilde{\boldsymbol{I}}_t^{(1),\ell}$ and $\tilde{\boldsymbol{H}}_t^{(2),\ell}$ with the last entries removed, and
		\begin{equation}
			\begin{split}
				\boldsymbol{e}_\parallel^\ell=(\boldsymbol{V}_t^\ell,\boldsymbol{Z}_t^\ell)(\boldsymbol{C}_t^\ell)^{-1}&\left(\begin{matrix}
					\boldsymbol{I}&0\\
					-\boldsymbol{\hat{M}}_{t+1}^{\beta,\ell}&\boldsymbol{I}
				\end{matrix}\right)\\
				&n_\ell^{-1}\left(\begin{matrix}
					(\boldsymbol{P}_t^\ell)^T\\
					(\boldsymbol{D}_t^\ell)^T\boldsymbol{\Lambda}_\ell
				\end{matrix}\right)\boldsymbol{\Lambda}_\ell\boldsymbol{d}_{t+1}^\ell,
			\end{split}
		\end{equation}
		where
		\begin{equation}
			\boldsymbol{C}_t^\ell=\frac{1}{n_\ell}\left(\begin{matrix}
				(\boldsymbol{V}_t^\ell)^T\boldsymbol{V}_t^\ell&(\boldsymbol{V}_t^\ell)^T\boldsymbol{Z}_t^\ell\\
				(\boldsymbol{Z}_t^\ell)^T\boldsymbol{V}_t^\ell&(\boldsymbol{Z}_t^\ell)^T\boldsymbol{Z}_t^\ell
			\end{matrix}\right).
		\end{equation}
		
		Let us step back to $2^{(c),L}$. According to $t^{(c),L}$, $\boldsymbol{Z}_1^L$ is composed of row-i.i.d.\ jointly Gaussian variables, and $\boldsymbol{V}_1^L$ is a function of $\boldsymbol{Z}_1^L$ and $\boldsymbol{F}_L$ ( independent of $\boldsymbol{Z}_1^L$). Hence, according to Stein's lemma,
		\begin{equation}
			\boldsymbol{C}_1^L\to\left(\begin{matrix}
				\tilde{\boldsymbol{\Gamma}}_1^L&\tilde{\boldsymbol{\Psi}}_1^L\tilde{\boldsymbol{\Omega}}_1^L\\
				\tilde{\boldsymbol{\Omega}}_1^L(\tilde{\boldsymbol{\Psi}}_1^L)^T&\tilde{\boldsymbol{\Omega}}_1^L
			\end{matrix}\right).
		\end{equation}
				Through arguments similar to \cite{fan2022approximate}, we can obtain
		\begin{equation}
			\boldsymbol{e}_\parallel^L\overset{W_2}{\to}V_1^L\tilde{\boldsymbol{\beta}}^L_1+Z_1^L(\tilde{\boldsymbol{\Omega}}^L_1)^{-1}\tilde{\boldsymbol{\omega}}_1^L,
		\end{equation}
		where $\tilde{\boldsymbol{\omega}}_1^L$ is the last column of $\tilde{\boldsymbol{\Omega}}_1^L$ with the last entry removed.
		
		As $\boldsymbol{z}_2^L=\boldsymbol{e}_2^L-\boldsymbol{V}_1^L\boldsymbol{\beta}_1^L$, the distribution of $\boldsymbol{z}_2^L$ can be obtained via
		\begin{equation}
			\boldsymbol{z}_2^L=\boldsymbol{z}_1^\ell(\boldsymbol{\tilde{\Omega}}^L_1)^{-1}\boldsymbol{\tilde{\omega}}_1^L+\boldsymbol{e}_\perp^L
		\end{equation}
		such that the empirical distribution of $(\boldsymbol{z}_1^L,\boldsymbol{z}_2^L)$ is asymptotically jointly Gaussian and independent from $\{M_1^\ell\}_{\ell=1}^L$ with covariance $\mathbb{E}[(\boldsymbol{z}_1^\ell)^T\boldsymbol{z}_2^\ell]=\tilde{\boldsymbol{\omega}}_1$ and $\mathbb{E}[(\boldsymbol{z}_2^\ell)^2]=[\tilde{\boldsymbol{\Omega}}]_{2,2}^L$.
		
		We then iterate the following points for $\ell=L-1,\ldots,2,1$.
		
		$1^{(g),\ell}$: As $\boldsymbol{u}_2^\ell=\tilde{h}_2^\ell(\boldsymbol{m}_1^\ell,\boldsymbol{z}_1^{\ell+1},\boldsymbol{z}_2^{\ell+1},\boldsymbol{F}_\ell)$, according to $1^{(f),\ell}$, $2^{(c),\ell+1}$ and Lemma \ref{useful lemma1}, we have
		\begin{equation}
			(\boldsymbol{u}_1^\ell,\boldsymbol{u}_2^\ell,\boldsymbol{m}_1^\ell,\boldsymbol{F}_\ell)\overset{W_2}{\to}(U_1^\ell,U_2^\ell,M_1^\ell,F_\ell)
		\end{equation}
		independent of $
		\{Z_1^\ell,V_1^\ell\}_{\ell=1}^L$.
		
		$2^{(a),\ell}$: This follows from Lemma B.4 of \cite{fan2022approximate}.
		
		$2^{(b),\ell}$: This follows from $2^{(b),L}$. According to $1^{(f),\ell}$ and $2^{(c),
			\ell+1}$, $\boldsymbol{M}_1^\ell$ is composed of jointly Gaussian distributed columns, and $\boldsymbol{U}_1^\ell$ is a function of $\boldsymbol{M}_1^\ell$, $\boldsymbol{E}_\ell$ and $\boldsymbol{z}_1^{\ell+1}$, $\boldsymbol{z}_2^{\ell+1}$, which are i.i.d.\ variables independent of $\boldsymbol{M}_1^\ell$. Hence, according to Stein's lemma, we will obtain similar results.
		
		$2^{(c),\ell}$: This follows from $2^{(c),L}$. According to $1^{(c),\ell}$, $\boldsymbol{Z}_1^\ell$ is composed of row-i.i.d.\ jointly Gaussian variables, and $\boldsymbol{V}_1^\ell$ is a function of $\boldsymbol{Z}_1^\ell$, $\boldsymbol{F}_\ell$ and $\boldsymbol{m}_1^{\ell-1}$ (independent of $\boldsymbol{Z}_1^\ell$ from $1^{(f),\ell-1}$). Hence, according to Stein's lemma, we will obtain similar results.
		
		Next, we will assume that Lemma 1 holds for iterations up to (including) $t$, and Lemma 1(a)-(f) hold for $t+1$ and layers up to (not including) $\ell$.
		
		$(t+1)^{(a),\ell} (t>2)$: This follows from Lemma B.4 of \cite{fan2022approximate}.
		
		$(t+1)^{(b),\ell} (t>2)$: This follows from $2^{(b),\ell}$, by which we have $\boldsymbol{d}_{t+1}^\ell=\boldsymbol{d}^\ell_\parallel+\boldsymbol{d}^\ell_\perp$ with $\boldsymbol{d}_\perp^\ell$ in \eqref{d_perp}
		and
		\begin{equation}
			\boldsymbol{d}^\ell_\parallel\overset{W_2}{\to}(D_1^\ell,\ldots,D_t^\ell,\bar{P}_1^\ell,\ldots,\bar{P}_t^\ell)(\boldsymbol{\gamma}_t^\ell)^{-1}\left(\begin{matrix}
				\tilde{\boldsymbol{\delta}}_t^\ell\\
				\tilde{\boldsymbol{\phi}}_t^\ell
			\end{matrix}\right),
		\end{equation}
		where $\mathbb{E}[(D_1^\ell,\ldots,D_t^\ell)^T(\Lambda_1^\ell)^{2k}(D_1^\ell,\ldots,D_t^\ell)]=\tilde{\boldsymbol{H}}_t^{(2k),\ell}$, $\mathbb{E}[(\bar{P}_1^\ell,\ldots,\bar{P}_{t-1}^\ell)^T(\Lambda_1^\ell)^{2k}(D_1^\ell,\ldots,D_t^\ell)]=\tilde{\boldsymbol{I}}_{t-1}^{\prime(2k),\ell}$.
		
		$(t+1)^{(c),\ell} (t>2)$: This follows from $2^{(c),\ell}$, by which we have $\boldsymbol{e}_{t+1}^\ell=\boldsymbol{e}_\parallel^\ell+\boldsymbol{e}_\perp^\ell$, with $\boldsymbol{e}_\perp^\ell$ in \eqref{e_perp} and
		\begin{equation}
			\boldsymbol{e}_\parallel^\ell\overset{W_2}{\to}(V_1^\ell,\ldots,V_t^\ell)\tilde{\boldsymbol{\beta}}^\ell_t+(Z_1^\ell,\ldots,Z_t^\ell)(\tilde{\boldsymbol{\Omega}}^\ell_t)^{-1}\tilde{\boldsymbol{\omega}}_t^\ell.
		\end{equation}
		
		As $\boldsymbol{z}_{t+1}^\ell=\boldsymbol{e}_{t+1}^\ell-\boldsymbol{V}_t^\ell\boldsymbol{\beta}_t^\ell$, the distribution of $\boldsymbol{z}_{t+1}^\ell$ can be obtained via
		\begin{equation}
			\boldsymbol{z}_{t+1}^\ell=(\boldsymbol{z}_1^\ell,\ldots,\boldsymbol{z}_t^\ell)(\tilde{\boldsymbol{\Omega}}^\ell_t)^{-1}\tilde{\boldsymbol{\omega}}_t^\ell+\boldsymbol{e}_\perp^\ell
		\end{equation}
		such that the empirical distribution of $(\boldsymbol{z}_1^\ell,\ldots,\boldsymbol{z}_{t+1}^\ell)$ is asymptotically jointly Gaussian and independent from $\{M_1^\ell,\ldots,M_t^\ell\}_{\ell=1}^L$ with covariance $\mathbb{E}[(\boldsymbol{z}_1^\ell,\ldots,\boldsymbol{z}_t^\ell)^T\boldsymbol{z}_{t+1}^\ell]=\tilde{\boldsymbol{\omega}}_t$ and $\mathbb{E}[(\boldsymbol{z}_{t+1}^\ell)^2]=[\tilde{\boldsymbol{\Omega}}]_{t+1,t+1}^\ell$.
		
		$(t+1)^{(d),\ell}$: This follows from Lemma B.4 of \cite{fan2022approximate}.
		
		$(t+1)^{(e),\ell}$: Note that the law of  $\boldsymbol{Q}_\ell$ is conditioned on
		\begin{equation}
			(\boldsymbol{P}_t^\ell,\boldsymbol{\Lambda}_\ell^T\boldsymbol{D}_{t+1}^\ell)\left(\begin{matrix}
				\boldsymbol{I}&-\hat{\boldsymbol{M}}_{t+1}^{\prime\beta,\ell}\\
				0&\boldsymbol{I}
			\end{matrix}\right)=\boldsymbol{Q}_\ell(\boldsymbol{V}_t^\ell,\boldsymbol{Z}_{t+1}^\ell),
		\end{equation}
		where $\hat{\boldsymbol{M}}_{t+1}^{\prime\beta,\ell}=\left(\begin{matrix}\hat{\boldsymbol{M}}_{t+1}^{\beta,\ell},\hat{\boldsymbol{\beta}}\end{matrix}\right)$. Thus, we can prove $(t+1)^{(e),\ell}$ similarly as before and as in \cite{fan2022approximate}.
		
		$(t+1)^{(f),\ell}$: Note that the law of $\boldsymbol{O}_\ell$ is conditioned on
		\begin{equation}
			(\boldsymbol{R}_{t+1}^\ell,\boldsymbol{\Lambda}_\ell\boldsymbol{P}_{t+1}^\ell)\left(\begin{matrix}
				\boldsymbol{I}&-\hat{\boldsymbol{M}}_{t+1}^{\prime\alpha,\ell}\\
				0&\boldsymbol{I}
			\end{matrix}\right)=\boldsymbol{O}_\ell(\boldsymbol{U}_{t+1}^\ell,\boldsymbol{M}_t^\ell),
		\end{equation}
		where $\hat{\boldsymbol{M}}_{t+1}^{\prime\alpha,\ell}=\left(\begin{matrix}\hat{\boldsymbol{M}}_{t+1}^{\alpha,\ell},\hat{\boldsymbol{\alpha}}\end{matrix}\right)$. Thus, we can obtain the asymptotic empirical distribution of $\boldsymbol{m}_{t+1}^\ell$ as
		\begin{equation}
			\boldsymbol{m}_{t+1}^\ell\overset{W_2}{\to}M_{t+1}^\ell=(M_1^\ell,\ldots,M_t^\ell)(\tilde{\boldsymbol{\Sigma}}_t^\ell)^{-1}\tilde{\boldsymbol{\sigma}}^\ell_t+M_\perp^\ell,
		\end{equation}
		where $M_\perp^\ell$ is a Gaussian distribution independent from previous steps:
		\begin{equation}
			\begin{aligned}
				M_\perp^\ell\sim&\mathcal{N}\left(0,\delta_\ell^{-1} [\tilde{\boldsymbol{L}}]_{t+1,t+1}^{(2)}\right.\\
				&-\left.\delta_\ell^{-2}
				\left(\begin{matrix}
					\tilde{\boldsymbol{j}}_t^{(1),\ell}\\\tilde{\boldsymbol{l}}_t^{(2),\ell}
				\end{matrix}\right)^T
				\left(\begin{matrix}
					\boldsymbol{\tilde{H}}_t^{(0),\ell}&(\tilde{\boldsymbol{I}}_t^{\prime(1),\ell})^t\\
					\tilde{\boldsymbol{I}}_t^{\prime(1),\ell}&\delta_\ell^{-1}\boldsymbol{L}_t^{(2),\ell}
				\end{matrix}\right)^{-1}
				\left(\begin{matrix}
					\tilde{\boldsymbol{j}}_t^{(1),\ell}\\\tilde{\boldsymbol{l}}_t^{(2),\ell}
				\end{matrix}\right)
				\right)
			\end{aligned}
		\end{equation}
		with $\tilde{\boldsymbol{j}}_t^{(1),\ell}$ and $\tilde{\boldsymbol{l}}_t^{(2),\ell}$ being the last rows of $\tilde{\boldsymbol{J}}_t^{(1),\ell}$ and $\tilde{\boldsymbol{L}}_t^{(2),\ell}$ except the last entries. According to $1^{(f),\ell}-t^{(f),\ell}$, we have
		\begin{equation}
			\begin{split}
				(\boldsymbol{u}_1^\ell,\ldots,\boldsymbol{u}_{t+1}^\ell,&\boldsymbol{m}_1^\ell,\ldots,\boldsymbol{m}_{t+1}^\ell,\boldsymbol{F}_\ell)\\
				&\overset{W_2}{\to}(U_1^\ell,\ldots,U_{t+1}^\ell,M_1^\ell,\ldots,M_{t+1}^\ell,F_\ell),
			\end{split}
		\end{equation}
		where the random variables on the RHS are independent of $\{Z_1^\ell,\ldots,Z_{t+1}^\ell\}_{\ell=1}^L$ and $\{V_1^\ell,\ldots,V_{t+1}^\ell\}_{\ell=1}^L$.
		
		Finally, we will assume that Lemma 1 holds for iterations up to (including) $t$ and Lemma 1(a)-(f) hold for $t+1$. We conclude the proof by iterating the following points for $\ell=1,2\ldots,L$.
		
		$(t+1)^{(g),\ell}$: Note that $$\boldsymbol{u}_{t+2}^\ell=\tilde{h}_{t+2}^\ell(\boldsymbol{m}_1^\ell,\ldots,\boldsymbol{m}_{t+1}^\ell,\boldsymbol{z}_1^{\ell+1},\ldots,\boldsymbol{z}_{t+1}^{\ell+1},\boldsymbol{F}_\ell).$$
		Thus, according to $1^{(f),\ell}-(t+1)^{(c),\ell}$, $1^{(c),\ell+1}-(t+1)^{(c),\ell+1}$, and Lemma \ref{useful lemma1}, we have
		\begin{equation}
			\begin{split}
				(\boldsymbol{u}_1^\ell,\ldots,\boldsymbol{u}_{t+2}^\ell,&\boldsymbol{m}_1^\ell,\ldots,\boldsymbol{m}_{t+1}^\ell,\boldsymbol{F}_\ell)\\
				&\overset{W_2}{\to}(U_1^\ell,\ldots,U_{t+2}^\ell,M_1^\ell,\ldots,M_{t+1}^\ell,F_\ell),
			\end{split}
		\end{equation}
		where the random variables on the RHS are independent of $\{Z_1^\ell,\ldots,Z_t^\ell\}_{\ell=1}^L$ and $\{V_1^\ell,\ldots,V_t^\ell\}_{\ell=1}^L$.
		
		$(t+1)^{(h),\ell}$: This follows from Lemma B.4 of \cite{fan2022approximate}.
		
	\end{proof}
	
	Note that Lemma 1(c) and Lemma 1(f) are equivalent to Theorem 2, so the proof of Theorem \ref{thm:gen} is complete.
	
	\subsection{Proof of Theorem \ref{thm:main}}\label{app:pfmain}
	To derive the ML-RI-GAMP algorithm from the general recursion, we first choose $E_1=X^1$, $E_\ell=\epsilon^\ell$, $F_\ell=\epsilon^{\ell+1}$. Second, we initialize the general recursion as
	\begin{equation}
		\begin{split}
			\boldsymbol{u}_1^\ell & =0,\quad \boldsymbol{z}_1^\ell=0,\quad \mbox{for }\ell\in [L],\\
			\boldsymbol{v}_1^1 & =\boldsymbol{x}^1,\quad \boldsymbol{v}_1^\ell=q^{\ell-1}(\boldsymbol{m}_1^{\ell-1},\boldsymbol{\epsilon}^\ell)=\boldsymbol{x}^\ell,\quad \mbox{for }\ell\in \{2, \ldots, L
			\},\\
			\boldsymbol{m}_1^\ell & =\boldsymbol{A}_\ell\boldsymbol{v}_1^\ell=\boldsymbol{g}^\ell,\quad \mbox{for }\ell\in [L],
		\end{split}
	\end{equation}
	that is
	\begin{equation}
		\tilde{f}_1^1(E_1)=x^1,
	\end{equation}
	\begin{equation}
		\tilde{f}_1^\ell(m_1^{\ell-1},E_\ell)=q^{\ell-1}(m_1^{\ell-1},\epsilon^\ell).
	\end{equation}
	Third, the non-linear functions are defined as
	\begin{equation}
		\begin{aligned}
			\tilde{f}_{t+1}^\ell&(v_1^\ell,z_1^\ell,\ldots,z_t^\ell,m_1^{\ell-1},\ldots,m_t^{\ell-1})\\&=f_t^\ell(z_2^\ell+\bar{\mu}_1^\ell v_1^\ell,\ldots,z_t^\ell+\bar{\mu}_{t-1}^\ell v_1^\ell,m_2^{\ell-1},\ldots,m_t^{\ell-1})\label{tildef}
		\end{aligned}
	\end{equation}
	for $t\geq1$ and
	\begin{equation}
		\tilde{h}_2^\ell(m_1^\ell,z_1^{\ell+1},z_2^{\ell+1},F^\ell)=h_1^\ell(z_2^{\ell+1}+\bar{\mu}_1^{\ell+1}v_1^{\ell+1}),
	\end{equation}
	\begin{equation}
		\tilde{h}_2^L(m_1^L,F^L)=h_1^L(q^L((m_1^L,\epsilon^{L+1})),
	\end{equation}
	\begin{equation}
		\begin{aligned}
			\tilde{h}_{t+1}^\ell&(m_1^\ell,\ldots,m_t^\ell,z_1^{\ell+1},\ldots,z_t^{\ell+1},F^\ell)\\
			&=h_t^\ell(m_2^\ell,\ldots,m_t^\ell,z_2^{\ell+1}+\bar{\mu}_1^{\ell+1}q^\ell(m_1^\ell,\epsilon^{\ell+1}),\ldots,\\
			&\hspace{10em}z_t^{\ell+1}+\bar{\mu}_{t-1}^{\ell+1}q^\ell(m_1^\ell,\epsilon^{\ell+1})),
		\end{aligned}
	\end{equation}
	\begin{equation}
		\tilde{h}_{t+1}^L(m_1^L,\ldots,m_t^L,F^L)=h_t^\ell(m_2^L,\ldots,m_t^L,q^L(m_1^L,\epsilon^{L+1})).
	\end{equation}
	for $t\geq2$.
	
	At this point, one can readily show that the state evolution of the general recursion and the state evolution in Algorithm \ref{algo:SE} are equivalent, as they are defined in a similar way.
	\begin{lemma}
		$\boldsymbol{\bar{\Psi}}_t^\ell=\boldsymbol{\tilde{\Psi}}_t^\ell$, $\boldsymbol{\bar{\Phi}}_t^\ell=\boldsymbol{\tilde{\Phi}}_t^\ell$ $\boldsymbol{\tilde{\Omega}}_t^\ell=\boldsymbol{\Omega}_t^{\prime\ell}$, $\boldsymbol{\tilde{\Sigma}}_t^\ell=\boldsymbol{\bar{\Sigma}}_t^\ell$.
	\end{lemma}
	Lemma 2 follows the proof of Lemma D.2 in \cite{venkataramanan2022estimation}, in which we can recursively prove the fact that
	\begin{equation}
		\begin{aligned}
			&(Z_2^\ell+\bar{\mu}_1X^\ell,\ldots,Z_t^\ell+\bar{\mu}_{t-1}X^\ell)\overset{d}{=}(X_1^\ell,\ldots,X^\ell_{t-1}),\\
			&(M_1^{\ell-1},\ldots,M_t^{\ell-1})\overset{d}{=}(R_1^{\ell-1},\ldots,R_{t-1}^{\ell-1}),
		\end{aligned}
	\end{equation}
	and by the definition of the non-linear function in \eqref{tildef},
	\begin{equation}
		\begin{aligned}
			\mathbb{E}[\partial_k\tilde{f}_{t+1}^\ell&(V_1^\ell,Z_1^\ell,\ldots,Z_t^\ell,M_1^{\ell-1},\ldots,M_t^{\ell-1},E_\ell)]\\&=\mathbb{E}[\partial_kf_t^\ell(X_1^\ell,\ldots,X_{t-1}^\ell,R_1^{\ell-1},\ldots,R_{t-1}^{\ell-1})].
		\end{aligned}
	\end{equation} 
	 
	Finally, Theorem \ref{thm:main} holds directly according to the following Lemma combined with Lemma \ref{lemma:main}.
	\begin{lemma}
		Let $\psi$: $\mathbb{R}^{2t+1}\to\mathbb{R}$ and $\phi$: $\mathbb{R}^{2t+2}\to\mathbb{R}$ be any pseudo-Lipschitz functions of order 2. Then for each $t\geq1$ and $\ell=1,\ldots,L$, we almost surely have
		\begin{equation}
			\begin{aligned}
				&\lim_{n_\ell\to\infty}\frac{1}{n_\ell}\sum_{i=1}^{n_\ell}|\psi(x_{1,i}^\ell,\ldots,x_{t,i}^\ell,\hat{x}_{1,i}^\ell,\ldots,\hat{x}_{t,i}^\ell,x_i^\ell)\\&-\psi(z_{2,i}^\ell+\bar{\mu}_1^\ell x_i^\ell,\ldots,z_{t+1,i}^\ell+\bar{\mu}_t^\ell x_i^\ell,v_{2,i}^\ell,\ldots,v_{t+1,i}^\ell,x_i^\ell)|=0,
			\end{aligned}
		\end{equation}
		\begin{equation}
			\begin{aligned}
				&\lim_{n_{\ell+1}\to\infty}\frac{1}{n_{\ell+1}}\sum_{i=1}^{n_{\ell+1}}|\phi(r_{1,i}^\ell,\ldots,r_{t,i}^\ell,s_{1,i}^\ell,\ldots,s_{t+1,i}^\ell,x_i^{\ell+1})\\&-\phi(m_{2,i}^\ell,\ldots,m_{t+1,i}^\ell,u_{2,i}^\ell,\ldots,u_{t+2,i}^\ell,q^\ell(m_{1,i}^\ell,\epsilon_i^{\ell+1}))|=0.
			\end{aligned}
		\end{equation}
	\end{lemma}
	\begin{proof}
		Following the proof of Lemma D.3 in \cite{venkataramanan2022estimation}, we only need to prove that 
		\begin{equation}
			\lim_{n_\ell\to\infty} \frac{1}{n_\ell}||\boldsymbol{x}_t^\ell-\boldsymbol{z}_{t+1}^\ell-\bar{\mu}_t^\ell\boldsymbol{x}^\ell||^2=0,
			\label{lemma3,eq1}
		\end{equation}
		\begin{equation}
			\lim_{n_{\ell+1}\to\infty} \frac{1}{n_\ell}||\boldsymbol{\hat{x}}_t^\ell-\boldsymbol{v}_{t+1}^\ell||^2=0,
			\label{lemma3,eq2}
		\end{equation}
		\begin{equation}
			\lim_{n_{\ell+1}\to\infty} \frac{1}{n_{\ell+1}}||\boldsymbol{r}_t^\ell-\boldsymbol{m}_{t+1}^\ell||^2=0,
			\label{lemma3,eq3}
		\end{equation}
		and
		\begin{equation}
			\lim_{n_{\ell+1}\to\infty} \frac{1}{n_{\ell+1}}||\boldsymbol{s}_{t+1}^\ell-\boldsymbol{u}_{t+2}^\ell||^2=0,
			\label{lemma3,eq4}
		\end{equation}
		hold almost surely for $t\ge 1$ and $\ell\in [L]$. We note that the almost sure boundedness can be proved in the same way, so we skip this part. Let us prove the claim by induction. We denote \eqref{lemma3,eq1}-\eqref{lemma3,eq4} as (a)-(d), and similarly $t^{(a),\ell}$ means \eqref{lemma3,eq1} for step $t$ and layer $\ell$.
		
		Our induction sequence follows the iteration in Algorithm \ref{algo:gen} as $1^{(a),L}\to1^{(a),L-1}\to\ldots\to1^{(a),1}\to1^{(b),1}\to1^{(c),1}\to\ldots\to1^{(b),L}\to1^{(c),L}\to1^{(d),1}\to1^{(d),2}\to\ldots\to1^{(d),L}\to\ldots\to t^{(a),1}\to t^{(a),2}\to\ldots\to t^{(a),L}\to t^{(b),1}\to t^{(c),1}\to\ldots\to t^{(b),L}\to t^{(c),L}\to t^{(d),1}\to t^{(d),2}\to\ldots\to t^{(d),L}$.
		
		$1^{(a),L}$: Remember that we initialize the algorithm as $\boldsymbol{u}_2^L=h_1^L(\boldsymbol{y})=\boldsymbol{s}_1^L$. Thus, we have
		\begin{equation}
			\boldsymbol{z}_2^L=\boldsymbol{A}_L^T\boldsymbol{s}_1^L-\delta_L\kappa_2^L\langle\partial_g\boldsymbol{s}_1^L\rangle\boldsymbol{x}^L.
		\end{equation}
		Note that $\boldsymbol{x}_1^L=\boldsymbol{A}_L^T\boldsymbol{s}_1^L$ and that $\bar{\mu}_1^L=\delta_L\bar{\kappa}_2^L\mathbb{E}[\partial_gS_1^L]$. As ${\kappa}_2^L\to\bar{\kappa}_2^L$ and $\langle\partial_g\boldsymbol{s}_1^L\rangle\to\mathbb{E}[\partial_GS_1^L]$, we obtain \eqref{lemma3,eq1} for $t=1$ and $\ell=L$.
		
		Next we assume that \eqref{lemma3,eq1} holds for $t=1$ and layers $L,L-1,\ldots,\ell+1$.
		
		$1^{(a),\ell}$: For the $\ell-$th layer,
		\begin{equation}
			\boldsymbol{z}_2^\ell=\boldsymbol{A}_\ell^T h_1^\ell(\boldsymbol{z}_2^{\ell+1}+\bar{\mu}_1^{\ell+1}\boldsymbol{v}_1^{\ell+1})-\delta_\ell\kappa_2^\ell\langle\partial_g\boldsymbol{s}_1^\ell\rangle\boldsymbol{x}^\ell.
		\end{equation}
		From $1^{(a),\ell+1}$ and Lemma \ref{useful lemma2}, $\frac{1}{n_{\ell+1}}||\boldsymbol{u}_2^\ell-\boldsymbol{s}_1^\ell||^2=\frac{1}{n_{\ell+1}}||h_1^\ell(\boldsymbol{z}_2^{\ell+1}+\bar{\mu}_1^{\ell+1}\boldsymbol{v}_1^{\ell+1})-h_1^\ell(\boldsymbol{x}_1^{\ell+1})||^2\to0$. We note that as $\Lambda_\ell$ has compact support indicated in Assumption 1(a), $\frac{1}{n_\ell}||\boldsymbol{A}_\ell^T(\boldsymbol{u}_2^\ell-\boldsymbol{s}_1^\ell)||$ goes to zero once $\frac{1}{n_{\ell+1}}||\boldsymbol{u}_2^\ell-\boldsymbol{x}_1^\ell||$ goes to zero. As $\bar{\mu}_1^\ell=\delta_\ell\bar{\kappa}_2^\ell\mathbb{E}[\partial_gS_1^\ell]$, ${\kappa}_2^\ell\to\bar{\kappa}_2^\ell$ and $\langle\partial_g\boldsymbol{s}_1^\ell\rangle\to\mathbb{E}[\partial_GS_1^\ell]$, we obtain \eqref{lemma3,eq1} for $t=1$ and all layers.
		
		$1^{(b),1}$: Since $\hat{\boldsymbol{x}}_1^1=f_1^1(\boldsymbol{x}_1^1)$, $\boldsymbol{v}_2^1=f_1^1(\boldsymbol{z}_2^1+\bar{\mu}_1^1\boldsymbol{x}^1)$ and $f_1^1$ is Lipschitz, we obtain \eqref{lemma3,eq2} for $t=1$ and $\ell=1$ due to Lemma \ref{useful lemma1} and $1^{(a),1}$.
		
		$1^{(c),1}$: Since
		\begin{equation}
			\boldsymbol{m}_2^1=\boldsymbol{A}_1\boldsymbol{v}_1^1-\kappa_2^1\langle\partial_1f_1^1(\boldsymbol{z}_2^1+\bar{\mu}_1^1\boldsymbol{x}^1)\rangle\boldsymbol{s}_1^1
		\end{equation} 
		and
		\begin{equation}
			\boldsymbol{r}_1^1=\boldsymbol{A}_1\hat{\boldsymbol{x}}_1^1-\bar{\kappa}_2^1\langle\partial_1f_1^1(\boldsymbol{x}_1^1)\rangle\boldsymbol{s}_1^1,
		\end{equation} 
		we can show $1^{(c),1}$ according to $1^{(b),1}$ and Lemma \ref{useful lemma2} due to the Lipschitz property of $f_1^1$.
		
		Next we assume that \eqref{lemma3,eq2} and \eqref{lemma3,eq3} hold for $t=1$ and layers up to (including) $\ell-1$.
		
		$1^{(b),\ell}$: Since $\boldsymbol{\hat{x}}_1^\ell=f_1^\ell(\boldsymbol{x}_1^\ell,\boldsymbol{r}_1^{\ell-1})$, $\boldsymbol{v}_2^\ell=f_1^\ell(\boldsymbol{z}_2^\ell+\bar{\mu}_1^\ell\boldsymbol{x}^\ell,\boldsymbol{m}_2^{\ell-1})$ and $f_1^\ell$ is Lipschitz, we obtain \eqref{lemma3,eq2} for $t=1$ and $\ell$ using $1^{(a),\ell}$, $1^{(c),\ell-1}$ and Lemma \ref{useful lemma1}.
		
		$1^{(c),\ell}$: Same argument as $1^{(c),1}$.
		
		Next we assume that \eqref{lemma3,eq1}-\eqref{lemma3,eq3} hold for $t=1$.
		
		$1^{(d),\ell}$: Since $\boldsymbol{s}_2^\ell=h_2^\ell(\boldsymbol{r}_1^\ell,\boldsymbol{x}_1^{\ell+1})$ and $\boldsymbol{u}_3^\ell=h_2^\ell(\boldsymbol{m}_2^\ell,\boldsymbol{z}_2^{\ell+1}+\bar{\mu}_1^{\ell+1}\boldsymbol{x}^{\ell+1})$, $1^{(d),\ell}$ holds according to Lemma \ref{useful lemma1} by using $1^{(c),\ell}$, $1^{(a),\ell+1}$ and the Lipschitz property of $h_2^\ell$. The case $\ell=L$ is similar.
		
		Next we assume that \eqref{lemma3,eq1}-\eqref{lemma3,eq4} hold for steps up to (including) $t-1$.
		
		$t^{(a),\ell}$: From the definition,
		\begin{equation}
			\begin{aligned}
				\boldsymbol{x}_t^\ell-\boldsymbol{z}_{t+1}^\ell&-\bar{\mu}_t^\ell\boldsymbol{x}^\ell=\boldsymbol{A}_\ell^T(\boldsymbol{s}_t^\ell-\boldsymbol{u}_{t+1}^\ell)+\sum_{i=1}^{t-1}\hat{\beta}_{t+1,i+1}^\ell(\boldsymbol{v}_{i+1}^\ell-\boldsymbol{\hat{x}}_i^\ell)\\
				&+\sum_{i=1}^{t-1}(\hat{\beta}_{t+1,i+1}^\ell-\beta_{t,i}^\ell)\boldsymbol{\hat{x}}_i^\ell+(\hat{\beta}_{t+1,1}^\ell-\bar{\mu}_t^\ell)\boldsymbol{x}^\ell.
			\end{aligned}
		\end{equation}
		
		The convergence of the first and the second term is due to $(t-1)^{(d),\ell}$ and $1^{(b),\ell}$ to $(t-1)^{(b),\ell}$. For the last two terms, we need to show the convergence of auxiliary matrices. By the induction hypothesis,
		\begin{equation}
			(\boldsymbol{x}_1^\ell,\ldots,\boldsymbol{x}_{t-1}^\ell)\overset{W_2}{\to}(X_1^\ell,\ldots,X_{t-1}^\ell),
		\end{equation}
		and thus $\boldsymbol{\Psi}_t^\ell\to\boldsymbol{\bar{\Psi}}_t^\ell$ according to Lemma \ref{useful lemma2}. Similarly, $\boldsymbol{\Phi}_{t+1}^\ell\to\boldsymbol{\bar{\Phi}}_{t+1}^\ell$. From the definition of $\boldsymbol{M}_{t+1}^{\beta,\ell}$ in \eqref{boldsymbolM_t+1}, we obtain
		\begin{equation}
			\boldsymbol{M}_{t+1}^{\beta,\ell}\to\delta_\ell\sum_{j=0}^{t+1}\bar{\kappa}_{2(j+1)}^\ell\boldsymbol{\bar{\Phi}}^\ell_{t+1}(\boldsymbol{\bar{\Psi}}^\ell_{t+1}\boldsymbol{\bar{\Phi}}^\ell_{t+1})^j.
		\end{equation}
		
		For the general recursion, we also have $\boldsymbol{\hat{\Psi}}_t^\ell\to\boldsymbol{\tilde{\Psi}}_t^\ell$ and $\boldsymbol{\hat{\Phi}}_{t+1}^\ell\to\boldsymbol{\tilde{\Phi}}_{t+1}^\ell$. Notice that by Lemma 2,
		$\boldsymbol{\bar{\Psi}}_t^\ell$, $\boldsymbol{\bar{\Phi}}_{t+1}^\ell$ and $\boldsymbol{\tilde{\Psi}}_t^\ell$, $\boldsymbol{\tilde{\Phi}}_{t+1}^\ell$ are exactly the same thing. This indicates that $\lim_{n_\ell\to\infty}\boldsymbol{M}_{t+1}^{\beta,\ell}=\lim_{n_\ell\to\infty}\boldsymbol{\hat{M}}_{t+1}^{\beta,\ell}$, and thus $\lim_{n_\ell\to\infty}|\hat{\beta}_{t+1,i+1}^\ell-\beta_{t,i}^\ell|=0$, $\lim_{n_\ell\to\infty}|\hat{\beta}_{t+1,1}^\ell-\bar{\mu}^\ell_t|=0$. Hence, the proof of $t^{(a),\ell}$ is complete.
		
		$t^{(b),1}$: As $\hat{\boldsymbol{x}}_t^1=f_t^1(\boldsymbol{x}_1^1,\ldots,\boldsymbol{x}_t^1)$ and $\boldsymbol{v}_{t+1}^1=f_t^1(\boldsymbol{z}_2^1+\bar{\mu}_1^1\boldsymbol{x}^1,\ldots,\boldsymbol{z}_{t+1}^1+\bar{\mu}_t^1\boldsymbol{x}^1)$ with $f_t^1$ Lipschitz, we can prove $t^{(b),1}$ according to Lemma \ref{useful lemma1}.
		
		$t^{(c),1}$: From the definition,
		\begin{equation}
			\begin{aligned}
				\boldsymbol{r}_t^1-\boldsymbol{m}_{t+1}^1=&\boldsymbol{A}_1(\boldsymbol{\hat{x}}_t^1-\boldsymbol{v}^1_{t+1})+\sum_{i=1}^t\hat{\alpha}^1_{t+1,i+1}(\boldsymbol{u}_{i+1}^1-\boldsymbol{s}_i^1)\\&+\sum_{i=1}^t(\hat{\alpha}^1_{t+1,i+1}-\alpha^1_{t,i})\boldsymbol{s}_i^1.
			\end{aligned}
		\end{equation}
		The convergence of the first and the second term is due to $(t-1)^{(b),1}$ and $1^{(d),1}$ to $(t-1)^{(d),1}$. The convergence of the last term follows from $t^{(a),\ell}$, as we can similarly show $\lim_{n_\ell\to\infty}|\boldsymbol{M}_{t+1}^{\alpha,1}-\hat{\boldsymbol{M}}_{t+1}^{\alpha,1}|=0$ and thus $\lim_{n_\ell\to\infty}|\hat{\alpha}_{t+1,i+1}^1-\alpha_{t,i}^1|=0$.
		
		Recursively using $t^{(c),\ell-1}$, we can show 
		$t^{(b),\ell}$, while $t^{(c),\ell}$ follows the proof of $t^{(c),1}$. 
		
		Lastly, we assume that \eqref{lemma3,eq1}-\eqref{lemma3,eq3} hold for  steps up to (including) $t$.
		
		$t^{(d),\ell}$: Since $\boldsymbol{s}_{t+1}^\ell=h_{t+1}^\ell(\boldsymbol{r}_1^\ell,\ldots,\boldsymbol{r}_t^\ell,\boldsymbol{x}_1^{\ell+1},\ldots,\boldsymbol{x}_t^{\ell+1})$ and $\boldsymbol{u}_{t+2}^\ell=h_{t+1}^\ell(\boldsymbol{m}_2^\ell,\ldots,\boldsymbol{m}_{t+1}^\ell,\boldsymbol{z}_2^{\ell+1}+\bar{\mu}_1^{\ell+1}\boldsymbol{x}^{\ell+1},\ldots,\boldsymbol{z}_{t+1}^{\ell+1}+\bar{\mu}_t^{\ell+1}\boldsymbol{x}^{\ell+1})$, $t^{(d),\ell}$ holds according to Lemma \ref{useful lemma1} by using $t^{(c),\ell}$, $t^{(a),\ell+1}$ and the Lipschitz property of $h_{t+1}^\ell$. The case $\ell=L$ is similar.
		
		This completes the proof of the lemma.
	\end{proof}
	\subsection{Useful lemmas}
	\begin{lemma}
		Let $f:\mathbb{R}^t\to\mathbb{R}$ be a Lipschitz function. If $(\boldsymbol{v}_1,\ldots,\boldsymbol{v}_t)\in\mathbb{R}^{n\times t}$ is a sequence of random variables that converges to $(V_1,\ldots,V_t)$ in Wasserstein-2 distance when $n\to\infty$, then $f(\boldsymbol{v}_1,\ldots,\boldsymbol{v}_t)\overset{W_2}{\to}f(V_1,\ldots,V_t)$. 
		\label{useful lemma1}
	\end{lemma}
	
	This is a special case of Proposition E.3 in \cite{fan2022approximate}.
	
	\begin{lemma}
		Let $f:\mathbb{R}^t\to\mathbb{R}$ be a Lipschitz function and $\partial_kf$ its derivative with respect to the k-th argument. If $(\boldsymbol{v}_1,\ldots,\boldsymbol{v}_t)\in\mathbb{R}^{n\times t}$ is a sequence of random variables that converges to $(V_1,\ldots,V_t)$ in Wasserstein-2 distance when $n\to\infty$, and the distribution
		of $(V_1, \ldots, V_t)$ is absolutely continuous with respect to the Lebesgue measure, then $\partial_kf(\boldsymbol{v}_1,\ldots,\boldsymbol{v}_t)\overset{W_2}{\to}\partial_kf(V_1,\ldots,V_t)$. 
		\label{useful lemma2}
	\end{lemma}
	
	This is Lemma E.1 in \cite{venkataramanan2022estimation} proved in \cite{feng2022unifying}.
	
	\begin{lemma}
		Let $\boldsymbol{O}\in\mathbb{R}^{n\times n}$ be a random Haar matrix and $\boldsymbol{v}\in\mathbb{R}^n$ be any deterministic vector satisfying $\frac{1}{n}||\boldsymbol{v}||^2\to\sigma^2$. Then, almost surely, $\boldsymbol{Ov}\overset{W_2}{\to}\mathcal{N}(0,\sigma^2)$.
		\label{useful lemma3}
	\end{lemma}
	
	This is a special case of Proposition F.2 in \cite{fan2022approximate}.
\end{document}